\documentclass[runningheads]{llncs}
\usepackage[T1]{fontenc}
\usepackage{graphicx}
\usepackage{xcolor}
\usepackage{amsmath}
\usepackage{amssymb}
\usepackage{enumitem}
\usepackage[linesnumbered,ruled,vlined]{algorithm2e}

\usepackage{xspace}
\usepackage{tabularx}
\usepackage{booktabs} 
\usepackage{multirow}
\usepackage{listings}
\usepackage{caption}
\usepackage{subcaption}
\usepackage{hyperref}
\usepackage{wrapfig}
\usepackage{comment}
\usepackage{graphicx}
\usepackage{adjustbox}

\newcommand{\sayan}[1]{\textcolor{blue}{#1}}

\newcommand{\xiangru}[1]{{\textcolor{orange}{#1}}}

\DeclareMathOperator*{\argsort}{arg\,sort}

\newcommand{\reals}{\mathbb{R}}

\newcommand{\add}{\mathrm{Add}}
\newcommand{\mul}{\mathrm{Mul}}
\newcommand{\mat}{\mathrm{Mmul}}
\newcommand{\inv}{\mathrm{Inv}}
\newcommand{\Prod}{\mathrm{Prod}}

\newcommand{\Exp}{\mathrm{Exp}}
\newcommand{\Sum}{\mathrm{Sum}}
\newcommand{\ind}{\mathrm{Ind}}

\newcommand{\sort}{\mathrm{Sort}}
\newcommand{\norm}{\mathrm{Norm}}
\newcommand{\pow}{\mathrm{Pow}}
\newcommand{\divi}{\mathrm{Div}}
\newcommand{\crown}{\text{CROWN}}

\newcommand{\ngcustom}{\mathrm{N}}

\newcommand{\gsplat}{\textsc{GaussianSplat}}
\newcommand{\absrend}{\textsc{AbstractSplat}}

\newcommand{\pcsort}{\textsc{BlendSort}}
\newcommand{\pcind}{\textsc{BlendInd}}
\newcommand{\matinv}{\textsc{MatrixInv}}
\newcommand{\pcsortmath}{\textsc{BlendSort2}}
\newcommand{\pcindmath}{\textsc{BlendInd2}}
\newcommand{\xpg}{\mathsf{XPG}}
\newcommand{\mpg}{\mathsf{MPG}}

\begin{document}
\title{Abstract Rendering: Computing All that is Seen in Gaussian Splat Scenes}
%
%
\author{Yangge Li, Chenxi Ji, Xiangru Zhong, Huan Zhang, Sayan Mitra}
\institute{University of Illinois at Urbana-Champaign 
Urbana, IL 61801, USA\\
}
\maketitle              
\begin{abstract}
We introduce abstract rendering, a method for computing a set of images by rendering a scene from a continuously varying range of camera positions. The resulting abstract image—which encodes an infinite collection of possible renderings—is represented using constraints on the image matrix, enabling rigorous uncertainty propagation through the rendering process. This capability is particularly valuable for the formal verification of vision-based autonomous systems and other safety-critical applications.
Our approach operates on Gaussian splat scenes, an emerging representation in computer vision and robotics. We leverage efficient piecewise linear bound propagation to abstract fundamental rendering operations, while addressing key challenges that arise in matrix inversion and depth sorting—two operations not directly amenable to standard approximations. To handle these, we develop novel linear relational abstractions that maintain precision while ensuring computational efficiency. These abstractions not only power our abstract rendering algorithm but also provide broadly applicable tools for other rendering problems.
Our implementation, $\absrend$, is optimized for scalability, handling up to 750k Gaussians while allowing users to balance memory and runtime through tile and batch-based computation. Compared to the only existing abstract image method for mesh-based scenes, $\absrend$ achieves $2-14\times$ speedups while preserving precision. Our results demonstrate that continuous camera motion, rotations, and scene variations can be rigorously analyzed at scale, making abstract rendering a powerful tool for uncertainty-aware vision applications.
\end{abstract}
\section{Abstract Rendering for Verification and Challenges}
\label{sec:intro}

Rendering computes an image from the description of a scene and a camera. In this paper, we introduce {\em abstract rendering}, which computes the set of {\em all} images that can be rendered from a set of scenes and a continuously varying range of camera poses. The resulting infinite collection of images, called an {\em abstract image}, is represented by constraints on the image matrix, such as using interval matrices. An example abstract image generated from the continuous movement of a camera is shown in Figure~\ref{fig:intro}.
Abstract rendering provides a powerful tool for the formal verification of autonomous systems, virtual reality environments, robotics, and other safety-critical applications that rely on computer vision. For instance, when verifying a vision-based autonomous system, uncertainties in the agent's state naturally translate into a set of camera positions. These uncertainties must be systematically propagated through the rendering process to analyze their impact on perception, control, and overall system safety. 
By enabling rigorous uncertainty propagation, abstract rendering supports  verification of statements such as ``No object is visible to the camera as it pans through a specified  angle.''

\begin{figure}[t]
    \centering
    \includegraphics[width=0.24\linewidth]{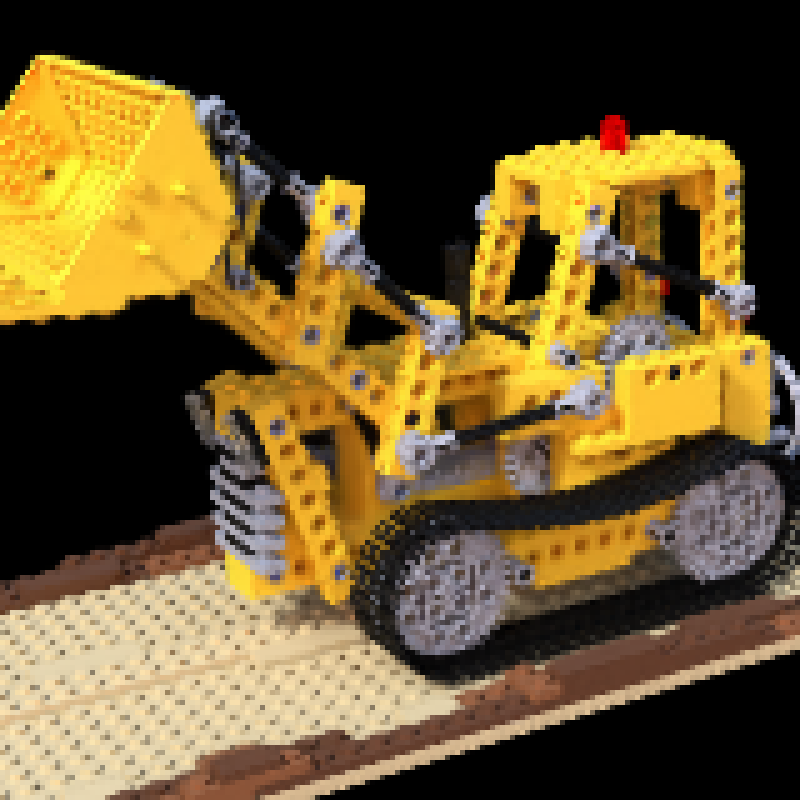}
    \includegraphics[width=0.24\linewidth]{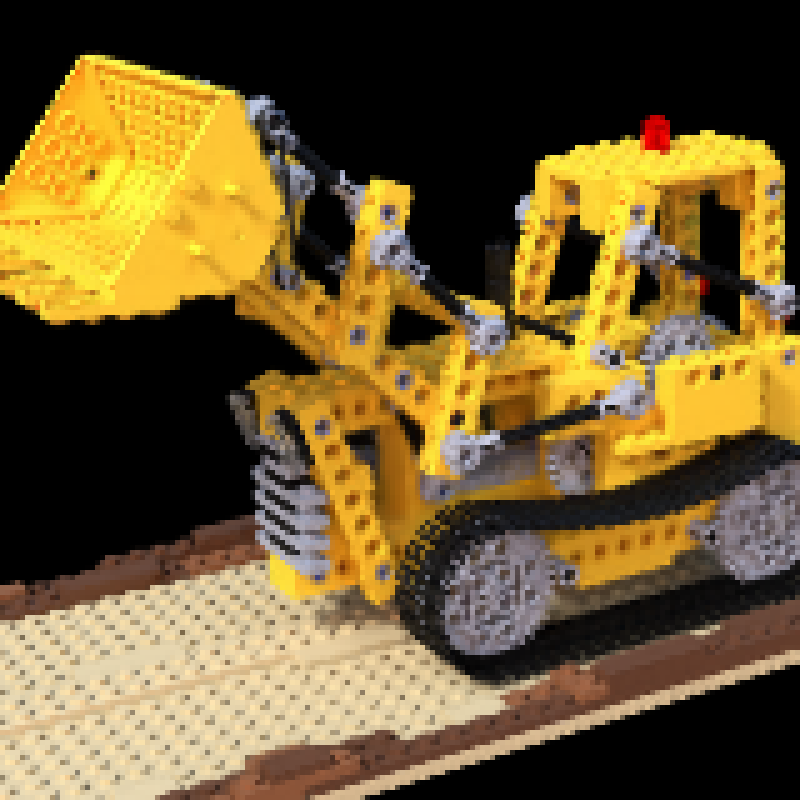}
    \includegraphics[width=0.24\linewidth]{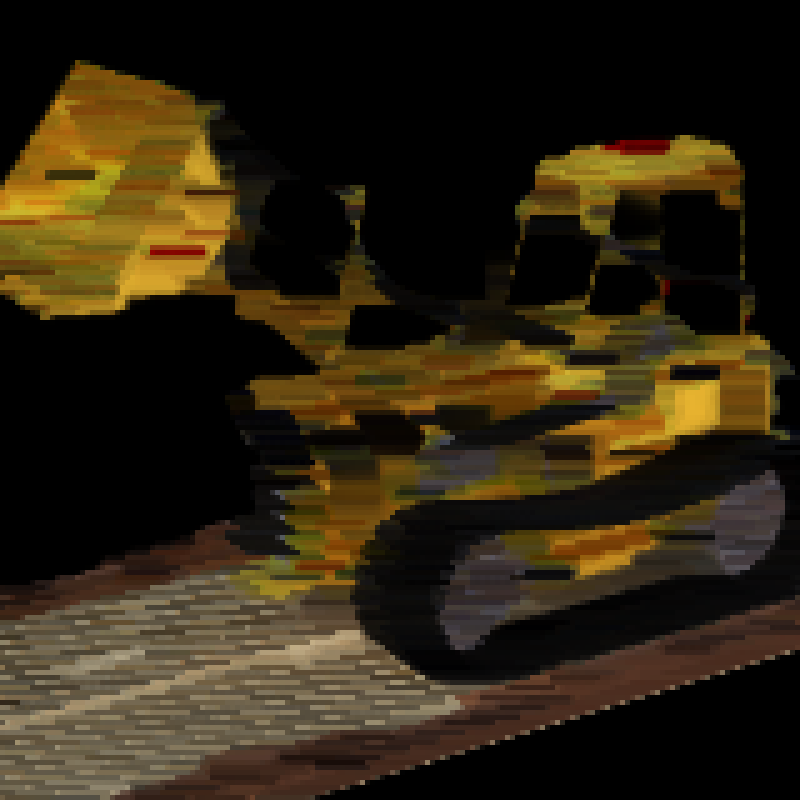}
    \includegraphics[width=0.24\linewidth]{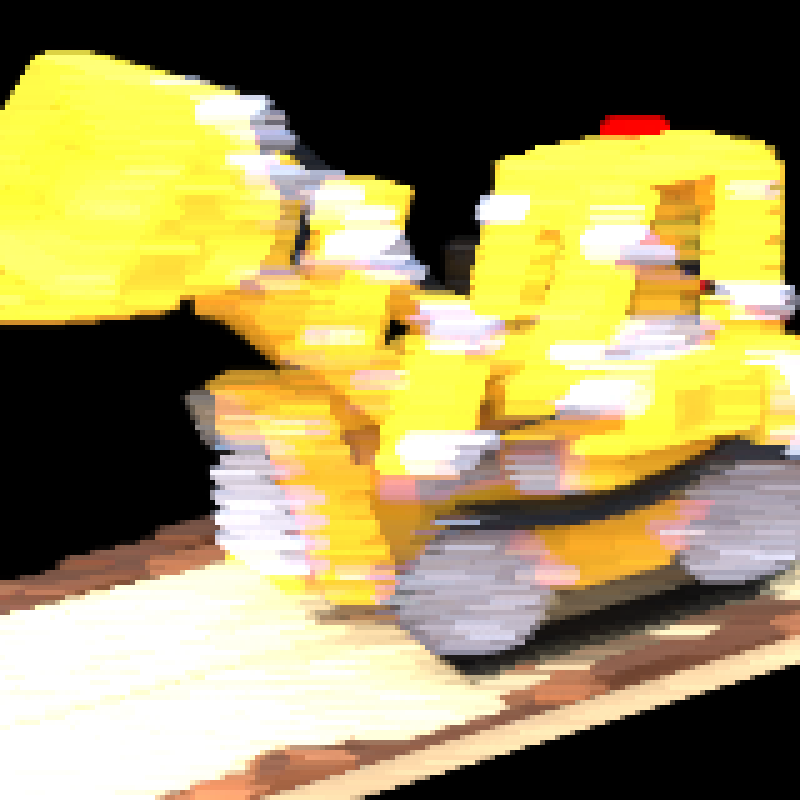}
    \caption{\small 
    Renderings of a Gaussian splat scene with a Lego bulldozer (from~\cite{mildenhall2020nerf}). Rendered image before ({\em Left}) and  after  camera shifts horizontally by 10cm ({\em Center left\/}). The lower ({\em Center right})  and upper ({\em Right}) bounds on  the abstract image computed by our abstract rendering method for the given camera movement.  For every pixel, the color in any image generated by the camera movement is guaranteed to lie between the corresponding pixel values in these lower and upper bound images.}
    \label{fig:intro}
\end{figure}

Our abstract rendering algorithm operates on scenes represented by Gaussian splats—an efficient and differentiable representation for 3D environments~\cite{10.1145/3592433}. A scene consists of a collection of 3D Gaussians, each with a distinct color and opacity. The rendering algorithm projects these splats onto the 2D image plane of a camera and blends them to determine the final color of each pixel. 
Gaussian splats are gaining prominence in computer vision and robotics due to their ability to be effectively learned from camera images and their capacity to render high-quality images in near real time~\cite{zhu20243dgaussiansplattingrobotics}. They have transformed virtual reality applications by enabling the rapid creation of detailed 3D scenes~\cite{jiang2024vr-gs,schieber2024semanticscontrolledgaussiansplattingoutdoor,anonymous2025gsvton}. 
In robotics and autonomous systems, Gaussian splatting has proven valuable for Simultaneous Localization and Mapping (SLAM) and scene reconstruction~\cite{tosi2024nerfs3dgaussiansplatting}.

To rigorously propagate uncertainty in such settings, we represent all inputs—such as camera parameters and scene descriptions—as linear sets (i.e., vectors or matrices defined by linear constraints). This abstraction allows us to systematically propagate uncertainty through each rendering operation.


The rendering algorithm for Gaussian splats computes the color of each pixel $\mathsf{u}$ in four steps. First, each 3D Gaussian is transformed from world coordinates to camera coordinates, and then to 2D pixel coordinates. The third step computes the {\em effective opacity} of each Gaussian at the pixel $\mathsf{u}$ by evaluating its probability density. Finally, the colors of all Gaussians are blended, weighted by their effective opacities and sorted by their distance to the image plane. 

The first two steps in the rendering process involve continuous transformations that can be precisely bounded above and below by piecewise linear functions. Although it is mathematically straightforward to propagate linear sets through these operations, achieving accurate bounds  for high-dimensional objects (e.g., scenes with over 100k Gaussians and images with over 10k pixels) requires a GPU-optimized linear bound propagation method such as the ones implemented in $\crown$~\cite{10.5555/3495724.3495820,zhang2018efficient}.

The third step evaluates Gaussian distributions over a linear set, which requires computing an interval matrix containing the inverses of multiple covariance matrices. Standard inversion methods (e.g., the adjugate approach) become numerically unstable near singularities and yield overly conservative bounds. We mitigate this by using a Taylor expansion around a stable reference—such as the midpoint of the input set—and by carefully tuning the expansion order and bounding the truncation error, we achieve inverse bounds that are sufficiently tight for abstract rendering.

The final step blends Gaussian colors based on depth. In abstract rendering, uncertain depths allow multiple possible sorting orders, and merging these orders leads to overly conservative bounds. For example, if several Gaussians with different colors can be frontmost, the resulting bound may unrealistically turn white. To address this, we replace depth-sorting with an index-free blending formulation that uses an indicator function to directly model pairwise occlusion, ensuring each Gaussian’s contribution is bounded independently.

Our abstract rendering algorithm, $\absrend$, integrates these techniques and supports tunable time–memory tradeoffs for diverse hardware. Our experiments show that $\absrend$ scales from single-object scenes to those with over 750k Gaussians—handling both camera translations and rotations—and achieves 2–14$\times$ speedups over the only existing abstract image method for mesh-based scenes~\cite{10745846} while preserving precision. Tile and batch optimizations enable users to balance GPU memory usage (which scales quadratically with tile size) against runtime (which scales linearly with batch size). Moreover, our experiments confirm that regularizing near-singular Gaussians is critical for reducing over-approximation. Together, these contributions significantly advance rigorous uncertainty quantification in rendering.


In summary, the key contributions from our paper are: (1) We present the first abstract rendering algorithm for Gaussian splats. It can rigorously bound the output of 3D Gaussian splatting under scene and camera pose uncertainty in both position and orientation. This enables formal analysis of vision-based systems in dynamic environments.
(2) We have developed a mathematical and a software framework for abstracting the key operations needed in the rendering algorithms. While many of these approximation can be directly handled by a library like $\crown$~\cite{10.5555/3495724.3495820}, we developed new methods for handling challenging operations like matrix inversion and sorting. We believe this framework can be useful beyond Gaussian splats to abstract other rendering algorithms. 
(3) We present a carefully engineered implementation that leverages tiled rendering and batched computation to achieve excellent scalability of abstract rendering. This design enables a tunable performance–memory trade-off, allowing our implementation to adapt effectively to diverse hardware configurations.

\paragraph{Relationship to Existing Research}
The only closely related work computes abstract images for triangular mesh scenes~\cite{10745846}.  
To handle camera-pose uncertainty, that method collects each pixel’s RGB values into intervals over all possible viewpoints. In practice, this amounts to computing partial interval depth-images for each triangle and merging them using a depth-based union, resulting in an interval image that over-approximates all concrete images for that region of poses. This work considers only camera translations, not rotations. Our method works with Gaussian splats which involve more complex operations and are learnable. Our algorithm  also has certain performance and accuracy advantages as discussed in Section~\ref{sec:exp:comp}.

Differentiable rendering~\cite{cgf.14507} algorithms are based on ray-casting~\cite{Lombardi:2019,mildenhall2020nerf,9156386} or  rasterization~\cite{10745846,10.1145/3415263.3419160,10.1145/3592433}. Although these techniques are amenable to gradient-based sensitivity analysis, they are not concerned with rigorous uncertainty propagation through the rendering process. We addresses this gap for Gaussian splatting. Our mathematical framework could be extended to other rasterization pipelines and even to ray-casting methods where similar numerical challenges arise.

Abstract interpretation~\cite{bouissou:inria-00528590,10.1007/978-3-540-75101-4_45} is effective for verifying large-scale software systems, however, applying it directly to rendering  presents unique challenges. Rendering involves high-dimensional geometric transformations (e.g., perspective projection, matrix inversion) and discontinuous operations (e.g., sorting) that exceed the typical scope of tools designed for control-flow analysis~\cite{10.1007/978-3-540-31987-0_3,10.1145/2737924.2738000} or neural network verification~\cite{10.1145/3290354,10.5555/3495724.3495820}. Our method adapts abstract interpretation principles specifically to the rendering domain, providing a tailored approach for bounding rendering-specific operations while balancing precision and scalability.

Finally, although neural network robustness verification~\cite{10.1145/3290354,tran2020cav,10.5555/3495724.3495820} analyzes how fixed image perturbations affect network outputs, our work operates upstream by generating the complete set of images a system may encounter under scene and camera uncertainty. This distinction is critical for safety-critical applications, where environmental variations—such as lighting and object motion—must be rigorously modeled before perception occurs.

\section{Background: Linear Approximations and Rendering}
\label{sec:background}
Our abstract rendering algorithm in Section~\ref{sec:abs-render} computes linear bounds on pixel colors derived from bounds on the scene and the camera parameters. In this section, we give an overview of the standard rendering algorithm for 3D Gaussian splat scenes
and introduce linear over-approximations of functions. 


\subsection{Linear Over-approximations of  Basic Operations}
\label{sec:lin_bg}
For a vector $x \in \reals^n$,  we denote its $i^{\mathit{th}}$  element by $x_i$. Similarly, for a matrix $A \in \reals^{m \times n}$, $A_{ij}$ 
is the element in the $i^{\mathit{th}}$ row and $j^{\mathit{th}}$ column.
For two vectors (or matrices) $x,y$ of the same shape, $x<y$ denotes element-wise comparison. For both a vector and matrix $x$, we denote $\Vert x\Vert$ as its Frobenius norm.
We use boldface to distinguish set-valued variables from their fixed-valued counterparts, (e.g., $\mathbf{uc}$ v.s. $\mathsf{uc}$).
A {\em linear set\/} or a polytope is a set $\{x \in \mathbb{R}^n \mid Ax \leq b\}$, where $Ax \leq b$ are the  constraints defining the set. An {\em interval  \/} is a special linear set.
A {\em piecewise linear relation \/} $R$ on $\reals^n \times \reals^m$ is a relation of the form $\{\langle x, y \rangle \in \mathbb{R}^n \times \mathbb{R}^m \mid \underline{A}x + \underline{b} \leq y \leq \overline{A}x + \overline{b}, Ax \leq b\}$, where $\langle \underline{A}, \underline{b}, \overline{A}, \overline{b} \rangle$ defines linear constraints on $y$ with respect to $x$. 
A {\em constant relation \/} is a special  piecewise linear relation where  $\underline{A}$ and $\overline{A}$ are zero matrices.
We define a class of functions that can be approximated arbitrarily precisely with piecewise linear lower and upper bounds over any compact input set. 
\begin{definition}
\label{def:lin-oa}
A function  $f:X\to \reals^m$ is called {\em linearly over-approximable} if, for any compact $B \subseteq X$ and $\varepsilon >0$, there exist  piecewise linear maps $\ell_B,u_B:B\to \reals^m$ such that for all $x\in B$, $\ell_B(x) \leq f(x) \leq u_B(x)$ and 
$|u_B(x) -  \ell_B(x)|< \varepsilon$. 
\end{definition}

For any linearly over-approximable function $y = f(x)$ and any linear set of inputs $B\subseteq X$, $R=\{(x,y)\mid\ell_B(x)\leq y\leq u_B(x),  x\in B\}$ is a piecewise linear relation that over-approximates function $f$ over $B$.



\begin{proposition}
\label{prop:cont-lin-approx}
Any continuous function is linearly over-approximable.
\end{proposition}


Although Proposition~\ref{prop:cont-lin-approx} implies piecewise linear bounds on continuous functions can be adjusted sufficiently tight by shrinking the neighborhood, finding accurate bounds over larger neighborhoods remains a significant challenge. Our implementation of abstract rendering uses \crown~\cite{10.5555/3495724.3495820} for computing such bounds.
Except for $\ind$, $\inv$, and $\sort$, all the basic operations listed in Table~\ref{tab:op} are continuous and therefore linearly over-approximable. Moreover, since $\ind$ is a piecewise constant function, it can be tightly bounded on each subdomain by partitioning at $x = 0$, allowing it to be treated as a linearly over-approximable function. This is formalized in Corollary~\ref{cor:op-lin-oa}.

\begin{corollary}
    \label{cor:op-lin-oa}
All operations in Table~\ref{tab:op}, except for $\inv$ and $\sort$, are linearly over-approximable. 
\end{corollary}


\begin{table}[t]
    \caption{\small Operations. 
     Conditional $A?B:C$  returns $B$ if $A$ is true and otherwise $C$. Permuting $x$ based on $y$ means reordering the elements of $x$ according to the indices that sort $y$ in ascending order. E.g., permuting $(9,3,7)$ based on $(5,13,8)$ results in $(9,7,3)$.} 
    \centering
    \begin{adjustbox}{max width=\textwidth}
    \begin{tabular}{|l|l|l|l|}
    \hline
    Operation  & Inputs & Output & Math Representation\\
    \hline
    Element-wise add ($\add$) & $x,y\in\reals^{n\times m}$ & $z\in\reals^{n\times m}$ & $z_{ij}=x_{ij}+y_{ij}$ \\
    \hline
    Element-wise multiply ($\mul$) & $x,y\in\reals^{n\times m}$ & $z\in\reals^{n\times m}$ & $z_{ij}=x_{ij}*y_{ij}$\\
    \hline
    Division ($\divi$) & $x\in\reals_{> 0}$ & $z\in\reals_{> 0}$ & $z= \frac{1}{x}$\\
    \hline 
    Matrix multiplication ($\mat$) & $x\in\reals^{n\times m},y\in\reals^{m\times k}$ & $z\in\reals^{n\times k}$ & $z = x \times y$\\
    \hline
    Matrix inverse ($\inv$) & $x\in\reals^{n\times n}$ & $z\in\reals^{n\times n}$ & $z = x^{-1}$\\
    \hline
    Matrix power ($\pow$) & $x\in\reals^{n\times n},k\in\mathbb{N}$ & $z\in\reals^{n \times n}$ & $z= x\times x\times \cdots \times x $\\
    \hline
    Summation ($\Sum$) & $x\in\reals^{n},k\in\mathbb{N}$ & $z\in\reals$ & $z=\sum_{i=1}^k x_{i}$ \\
    \hline
    Product ($\Prod$) & $x\in\reals^{n},k\in\mathbb{N}$ & $z\in\reals$ & $z=\prod_{i=1}^k x_{i}$ \\
    \hline
    Matrix transpose ($\top$) & $x\in\reals^{n\times m}$ & $z\in\reals^{m\times n}$ & $z_{ij}=x_{ji}$\\
    \hline
    Element-wise exponential ($\Exp$) & $x\in\reals^{n\times m}$ & $z\in\reals^{m\times n}$ & $z_{ij}=e^{x_{ij}}$\\
    \hline
    Frobenius norm ($\norm$) & $x\in\reals^{n\times m}$ & $z\in\reals$ & $z=\Vert x\Vert$\\
    \hline
    Element-wise indicator  ($\ind$) & $x\in\reals^{n\times m}$ & $z\in\reals^{n\times m}$ & $z_{ij}= (x_{ij}>0)? 1 : 0$\\
    \hline
    Sorting ($\sort$) & 
    $z\in\reals^{n}$ & $x\in\reals^{n}, y\in\reals^{n}$ & permute $x$ based on $y$
    \\
    \hline
    \end{tabular}
    \end{adjustbox}
    \\
    %
    \label{tab:op}
\end{table}
\vspace{-0.2cm}

Given an input interval $X=[\underline{x},\overline{x}]\subseteq \reals$, piecewise linear bounds and corresponding piecewise linear relations for operators $\Exp$, $\divi$ and $\ind$ are shown in Table~\ref{tab:lin-exm}. These bounds confirm these operations' linearly over-approximability. 

\begin{table}[t]
    \caption{\small Bounds for operations $\Exp$, $\divi$ and $\ind$.  Corresponding piecewise linear relations can be written as $R=\{(x,f(x))|\ell_X(x)\leq f(x)\leq u_X(x), \underline{x}\leq x\leq \overline{x}\}$.}
    \centering
    \begin{tabular}{|c|c|c|}
        \hline
       Operation $f(x)$ & Linear lower bound $\ell_{X}$ & Linear upper bound $u_{X}$ \\
       \hline
       $\Exp(x)$ & $\exp(\underline{x})(x-\underline{x})+\exp(\underline{x})$ & $\frac{\exp(\overline{x})-\exp(\underline{x})}{\overline{x}-\underline{x}}(x-\underline{x})+\exp(\underline{x})$ \\
       \hline
       $\divi(x)$ & $-\frac{1}{\underline{x}^2}(x-\underline{x})+\frac{1}{\underline{x}}$ & $-\frac{1}{\underline{x}\overline{x}}(x-\underline{x})+\frac{1}{\underline{x}}$\\
       \hline
       $\ind(x)$ & $\left\{\begin{matrix}
        0 & \mbox{ if } x\leq 0\\
        1 & \mbox{ if } x>0
    \end{matrix}\right.$ & $\left\{\begin{matrix}
        0 & \mbox{ if } x\leq 0\\
        1 & \mbox{ if }x> 0
    \end{matrix}
    \right.$\\
        \hline       
    \end{tabular}
    \label{tab:lin-exm}
\end{table}

\subsection{Rendering Gaussian Splat Scenes}
\label{sec:render_bg}

A Gaussian splat scene models a three dimensional scene as a collection of Gaussians with different means, covariances, colors, and transparencies.  
The  rendering algorithm for Gaussian splat scenes  projects these 3D Gaussians as 2D splats on the camera's image plane~\cite{10.1145/3592433}. 

Given a world coordinate frame $w$, 
a 3D Gaussian is  a tuple $\langle \sf{uw}, \sf{Mw}, \sf{o}, \sf{c} \rangle$ where,
$\sf{uw} \in \reals^{3}$ is its {\em mean} position in the world coordinates, 
$\sf{Mw}  \in \reals^{3\times3}$ is its covariance matrix\footnote{The covariance matrix $\sf{Cov}$ is represented by its 
Cholesky factor $\sf{Mw}$, i.e., $\sf{Cov}=\sf{Mw} (\sf{Mw})^T$.}, 
$\sf{o}\in [0,1]$ is its opacity, and 
$\sf{c} \in [0,1]^{3}$ is its colors in RGB channels. 
A {\em scene\/} $\mathsf{Sc}$ is a collection of Gaussians in the world coordinate frame, indexed by a finite set $\mathsf{I}$. A visualization of a scene with three Gaussians is shown on the right side of Figure~\ref{fig:3D-gsplat}.

\begin{figure}[t]
    \centering
    \includegraphics[width=0.75\linewidth]{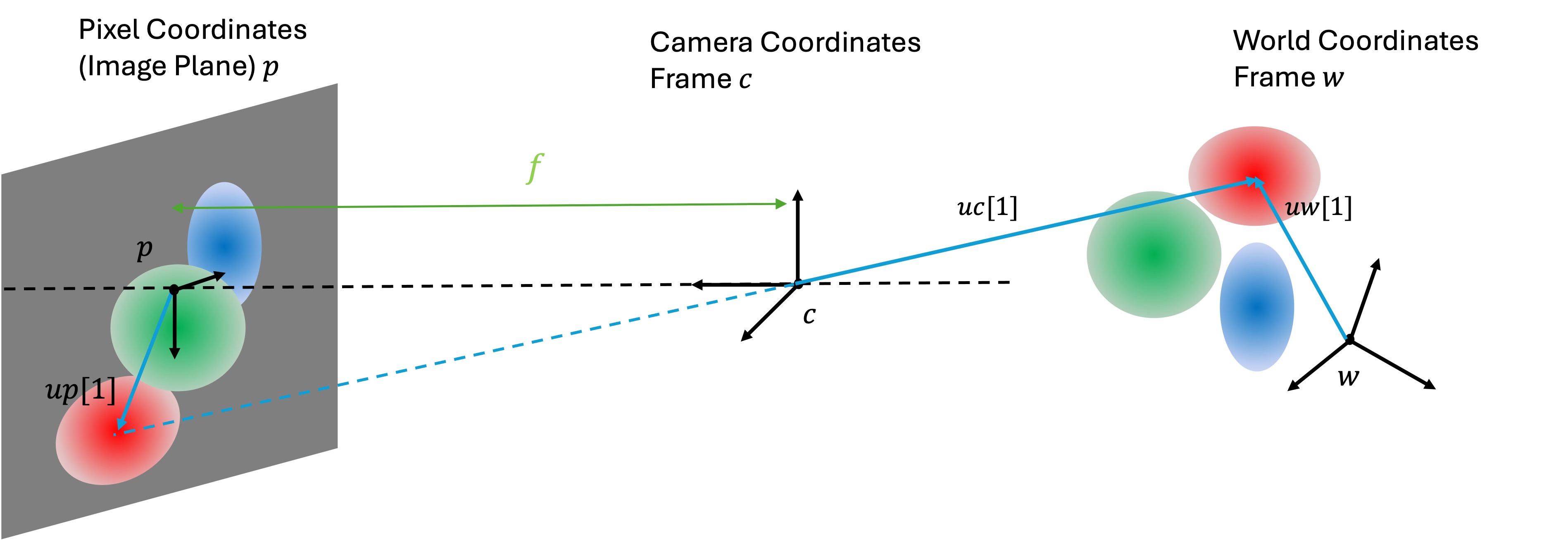}
    \caption{\small 
    Visualization of a 3D Gaussian Scene: Gaussians are defined in world coordinates (right), transformed into camera coordinates (middle), and projected onto the image plane/pixel coordinates (left). Here, $f$ denotes the focal length, and $\mathsf{uw}[1]$, $\mathsf{uc}[1]$, and $\mathsf{up}[1]$ represent the red Gaussian's mean in the respective coordinate systems.}
    \label{fig:3D-gsplat}
\end{figure}

A {\em camera} $C$ is defined by a translation $\sf{t}\in\reals^3$ and a rotation matrix $\sf{R}\in\reals^{3\times 3}$ relative to the world coordinates, a focal length $(\mathsf{f_x, f_y}) \in \reals^2$, and an  offset $(\mathsf{c_x,c_y}) \in \reals^2$. The position and the rotation  define the extrinsic matrix of the camera, and the latter parameters define the intrinsic matrix. Together, they define how 3D objects in the world are transformed to 2D pixels on the camera's image plane. An {\em Image} of size $\sf{W} \times \sf{H}$ (denoted as $\sf{WH}$) is defined by a collection of  pixel colors, where each pixel color $\sf{pc} \in [0,1]^{3}$.

Given a  scene $\mathsf{Sc}$ and a camera $\mathsf{C}$, the image seen by the camera is obtained by projecting the 3D Gaussians in the scene onto the camera's image plane and blending their colors. 
Each pixel's color is blended based on the distance between the pixel's position and the 2D Gaussian's mean, as well as the depth ordering of the 3D Gaussians from the image plane.
In detail, the rendering algorithm $\gsplat$ (shown in Listing~\ref{alg:gsplat}) computes the color of each pixel $\mathsf{u} \in \sf{WH}$ in the following four steps. The first three steps are for each Gaussian and the final step is on the aggregate. 
\begin{description}
\item[Step 1.] Transforms each 3D Gaussian $\sf{i} \in \mathsf{I}$ in the scene $\mathsf{Sc}$ from the world coordinates ($w)$ into camera coordinates ($c$) (Line~\ref{alg1:l2}-\ref{alg1:l3}). This is done by applying the rotation $\mathsf{R}$ and the translation $\mathsf{t}$ to the mean $\mathsf{uw[i]}$ and the rotation to the covariance $\mathsf{Mw[i]}$.
\item[Step 2.] Transforms each of the Gaussians $\sf{i}$ from 3D camera coordinates ($c$) to 2D pixel coordinates ($p$) on the image plane (Line~\ref{alg1:l4}-\ref{alg1:l7}) (Figure~\ref{fig:3D-gsplat}) by applying the intrinsic matrix  
$\mathsf{K}=\begin{bmatrix}
    \mathsf{f_x} & 0 &  \mathsf{c_x}\\
    0 & \mathsf{f_y} & \mathsf{c_y}
\end{bmatrix}$ to $\mathsf{uc[i]}$
and the projective transformation $\sf{J[i]}$ to $\sf{Mc[i]}$. And computes depth $\sf{d[i]}$ to image plane for Gaussian $\sf{i}$  simply as the third coordinate of $\mathsf{uc[i]}$ (as  the camera axis is orthogonal to the image plane).
\item[Step 3.] 
Next, computes the inverse of covariance $\sf{Conic[i]}$ for the $\sf{i^{th}}$ 2D Gaussian  in Step~2 (Line~\ref{alg1:l8}). 
Effective opacity of each Gaussian $\sf{a[i]}$ at pixel position $\sf{u}$ is computed as opacity $\sf{o[i]}$ weighed by the probability density of the $\sf{i^{th}}$ 2D Gaussian at pixel $\sf{u}$ (Line~\ref{alg1:l9}-\ref{alg1:l10}). 
The complicated expression for $\sf{a[i]}$ is the familiar definition of the Gaussian distribution with covariance $\sf{Mp[i]}$ written using the operators in Table~\ref{tab:op}.
\item[Step 4.]  Finally, 
the color of the pixel $\sf{pc}$ is determined by blending the colors of all the Gaussian weighted by their effective opacities $\sf{a}$ and adjusted according to their depth to the image plane $\sf{d}$ (Line~\ref{alg1:l11}).
\end{description}

\begin{algorithm}[t]
    \caption{$\gsplat (\mathsf{Sc},\mathsf{C},\mathsf{u})$}
    \label{alg:gsplat}
    
    \SetKwProg{Fn}{Function}{:}{}
    \For{$\forall \mathsf{i} \in \mathsf{I}$}{
        $\mathsf{uc}[\mathsf{i}] \gets \mat(\mathsf{R}, \add(\mathsf{uw}[\mathsf{i}], -\mathsf{t}))$\; \label{alg1:l2}
        $\mathsf{Mc}[\mathsf{i}] \gets \mat(\mathsf{R}, \mathsf{Mw}[\mathsf{i}])$\; \label{alg1:l3} 
        $\mathsf{J}[\mathsf{i}] \gets \begin{bmatrix}
            \mul(\mathsf{f_x}, \mathsf{uc}[\mathsf{i},2]) & 0 & -\mul(\mathsf{f_x}, \mathsf{uc}[\mathsf{i},0])\\
            0 & \mul(\mathsf{f_y}, \mathsf{uc}[\mathsf{i},2]) & -\mul(\mathsf{f_y}, \mathsf{uc}[\mathsf{i},1])
        \end{bmatrix}$\; \label{alg1:l4}
        $\mathsf{up}[\mathsf{i}] \gets \mat(\mathsf{K}, \mathsf{uc}[\mathsf{i}])$\; \label{alg1:l5}
        $\mathsf{Mp}[\mathsf{i}] \gets \mat(\mathsf{J}[\mathsf{i}], \mathsf{Mc}[\mathsf{i}])$\; \label{alg1:l6}
        $\mathsf{d}[\mathsf{i}] \gets \mathsf{uc}[\mathsf{i},2]$\; \label{alg1:l7}
        $\mathsf{Conic}[\mathsf{i}] \gets \inv(\mat(\mathsf{Mp}[\mathsf{i}], \mathsf{Mp}[\mathsf{i}]^\top))$\; \label{alg1:l8}
        $\mathsf{q}[\mathsf{i}] \gets \mat(\add(\mul(\mathsf{d}[\mathsf{i}], \mathsf{d}[\mathsf{i}], \mathsf{u}), -\mul(\mathsf{d}[\mathsf{i}], \mathsf{up}[\mathsf{i}])), \mathsf{Conic}[\mathsf{i}], \mathsf{Mp}[\mathsf{i}])$\; \label{alg1:l9}
        $\mathsf{a}[\mathsf{i}] \gets \mat(\mathsf{o}[\mathsf{i}], \Exp(\mat(-\frac{1}{2}, \mathsf{q}[\mathsf{i}], \mathsf{q}[\mathsf{i}]^\top)))$\; \label{alg1:l10}
    }
    $\mathsf{pc} \gets \pcsort(\mathsf{a}, \mathsf{c}, \mathsf{d})$\; \label{alg1:l11}
    \Return{$\mathsf{pc}$}\;
\end{algorithm}

This last step is implemented in the $\pcsort$ (shown in Listing~\ref{alg:pcsort}) function: it  sorts the  Gaussians based on their depth to the image plane $\sf{d}$, resulting in reordered effective opacities $\sf{as}$ and Gaussian colors  $\mathsf{cs}$ (Line~\ref{lin:alg-pcsort-1}-\ref{lin:alg-pcsort-2}). Then for each Gaussian, its transparency $\sf{vs[i]}$ is computed as $1-\sf{as[i]}$ (Line ~\ref{lin:alg-pcsort-3}). The transmittance of $\sf{i^{th}}$ Gaussian is obtained as the cumulative product of transparency of all Gaussians in front of it (Line ~\ref{lin:alg-pcsort-4}). Finally, the pixel color $\mathsf{pc}$ is determined by summation of all Gaussian colors $\sf{cs}$, weighted by their respective transmittance $\mathsf{Ts}$ and effective opacities $\mathsf{as}$ (Line ~\ref{lin:alg-pcsort-5}).

For rendering the whole image, $\gsplat$ is executed for each pixel. Although, this algorithm looks involved, we note that Steps~1 and~2 only use the $\mat$, $\add$, $\mul$ operations in Table~\ref{tab:op}.
Step~3 involves computing the inverse of the covariance matrix and Step~4 requires sorting the depth.

\begin{table}
\centering
\begin{minipage}[t]{0.40\textwidth}
  \vspace*{0pt} 
  \begin{algorithm}[H]
    \caption{$\pcsort(\mathsf{a},\mathsf{c},\mathsf{d})$}
    \label{alg:pcsort}
    $\mathsf{as}\gets\sort (\mathsf{a},\mathsf{d})$\; \label{lin:alg-pcsort-1} 
    $\mathsf{cs}\gets\sort (\mathsf{c},\mathsf{d})$\; \label{lin:alg-pcsort-2} 
    \For{$\forall \mathsf{i}\in \mathsf{I}$}{
      $\mathsf{vs}[\mathsf{i}]\gets\add(1,-\mathsf{as}[\mathsf{i}])$ \;\label{lin:alg-pcsort-3} 
      $\mathsf{Ts}[\mathsf{i}]\gets\Prod(\mathsf{vs},\mathsf{i}-1)$\;\label{lin:alg-pcsort-4}
    }
    $\mathsf{pc}\gets\Sum(\mul(\mathsf{Ts},\mathsf{as},\mathsf{cs}))$\; \label{lin:alg-pcsort-5}
    \Return{$\mathsf{pc}$}\;
  \end{algorithm}
\end{minipage}
\hfill
\begin{minipage}[t]{0.59\textwidth}
  \vspace*{0pt} 
  \begin{algorithm}[H]
    \caption{$\pcind(\mathsf{a},\mathsf{c},\mathsf{d})$}
    \label{alg:pcind}
    \For{$\forall \mathsf{i}\in \mathsf{I}$}{
      \For{$\forall \mathsf{j}\in \mathsf{I}$}{
        $\mathsf{V}[\mathsf{i},\mathsf{j}]\gets \add(1,-\mul(a[\mathsf{j}],\ind(\mathsf{d}[\mathsf{i}]-\mathsf{d}[\mathsf{j}])))$\; \label{lin:alg-pcind-1}
      }
      $ \mathsf{T}[\mathsf{i}]\gets\Prod(\mathsf{V}[\mathsf{i}],\ngcustom)$\; \label{lin:alg-pcind-2}
    }
    $\mathsf{pc}\gets\Sum(\mul(\mathsf{T},\mathsf{a} ,\mathsf{c}),\ngcustom)$\; \label{lin:alg-pcind-3}
    \Return{$\mathsf{pc}$}\;
  \end{algorithm}
\end{minipage}
\end{table}
\vspace{-0.2cm}

\paragraph{Abstract rendering problem.} Our goal is to design an algorithm that for all pixel $\mathsf{u}\in \mathsf{WH}$, it takes as input a linear set of scenes $\mathsf{\mathbf{Sc}}$ and  cameras $\mathsf{\mathbf{C}}$, and outputs a linear set of pixel color $\mathsf{\mathbf{pc}}$ such that 
$\gsplat(\mathsf{Sc},\mathsf{C},\mathsf{u}) \in \mathsf{\mathbf{pc}}$ 
, for each 
$\mathsf{Sc}\in\mathsf{\mathbf{Sc}}$ and 
$\mathsf{C}\in\mathsf{\mathbf{C}}$.

\section{Abstract  Rendering}
\label{sec:abs-render}
Our approach for creating the abstract rendering algorithm, $\absrend$, takes a linear set of scenes and cameras as input and step-by-step derives the piecewise linear relation between each intermediate variable and input through $\gsplat$.
This process continues until the final step, where the piecewise linear relation between pixel values and inputs is established. The final piecewise linear relation is then used to compute linear bounds on each pixel, and by iterating over all pixels in the image, we obtain linear bounds for the entire image.

In $\gsplat$, the first two steps involve linear operations, making it straightforward to derive the linear relation at each step as a combination of the current relation and the linear bounds of the ongoing operations. However, matrix inversion in step 3, using the adjugate method\footnote{In the adjugate method, the matrix inverse is computed by dividing the adjugate matrix by the determinant}, can introduce large over-approximation errors due to uncertainty in the denominator, particularly when the matrix approaches singularity. Step 4 further exacerbates these errors with sorting, as depth order uncertainty may cause a single Gaussian to be mapped to multiple indices, leading to over-counting in the pixel color summation, and potentially causing pixel color overflow in extreme cases.


To address these issues, we first introduce a Taylor expansion-based Algorithm $\matinv$ in Section~\ref{sec:inv}, capable of mitigating the uncertainty of division to be arbitrarily small by setting a sufficiently high Taylor order. Then in Section~\ref{sec:blending}, we propose $\pcind$ (shown in Listing~\ref{alg:pcind}) in replacement to $\pcsort$, which eliminates indices uncertainty and improves bound tightness by implementing  Gaussian color blending without the need for $\sort$ operation.
Finally, we present a detailed explanation of $\absrend$ and prove its soundness in Section~\ref{sec:ab-rendering}.

\subsection{Abstracting  Matrix Inverse by Taylor Expansion}
\label{sec:inv}

Linear bounds on division often lead to greater over-approximation errors compared to addition or multiplication. Thus, we propose a Taylor expansion-based algorithm, $\matinv$ (shown in Listing~\ref{alg:matinv}), which expresses the matrix inverse as a series involving only addition and multiplication. Though division is used to bound the remainder term, its impact can be minimized by increasing the Taylor order, effectively reducing the remainder to near zero.

\begin{algorithm}
    \caption{$\matinv(\mathsf{X}; \mathsf{X0}, \mathsf{k})$}
    \label{alg:matinv}
    \SetKwProg{Fn}{Function}{:}{}
    
    $\mathsf{IXX0}\gets\add(I,-\mat(\mathsf{X},\mathsf{X0}))$\;
    \label{lin:alg-matinv-1} 

    Assert $\norm(\mathsf{IXX0})<1$\;
    \label{lin:alg-matinv-2}

    \For{$
    \forall \mathsf{i}\in \{0,1,\cdots, \sf{k}\} $}{

     $\sf{Xa}=\mat(\mathsf{X0},\pow(\sf{IXX0},\sf{i}))$
     \label{lin:alg-matinv-3}
     
    }
    $\mathsf{Xp}\gets\Sum(\sf{Xa},\sf{k}+1)$\;
    \label{lin:alg-matinv-4} 
    
     $\mathsf{Eps}\gets\mul(\norm(\mathsf{X0}),\pow(\norm(\mathsf{IXX0}),k+1),\divi(\add(1,-\norm(\mathsf{IXX0}))))$\;
     \label{lin:alg-matinv-5}

    $\mathsf{lXinv}\gets\add(\mathsf{Xp},-\mathsf{Eps})$\;
    \label{lin:alg-matinv-6} 

    $\mathsf{uXinv}\gets \add(\mathsf{\mathsf{Xp},\mathsf{Eps}})$\;
    \label{lin:alg-matinv-7} 
    \Return{$\langle\mathsf{lXinv},\mathsf{uXinv}\rangle$}\;
\end{algorithm}

%
The $\matinv$ function takes non-singular matrix $\mathsf{X}$ as input and uses a reference matrix $\mathsf{X0}$ along with a Taylor expansion order $\mathsf{k}$ as parameters. It outputs $\mathsf{lXinv}$ and $\mathsf{uXinv}$, representing the lower and upper bounds of $\inv$ operation. 
Given fixed input $\sf{X}$, Algorithm $\matinv$ operates as follows: First, it calculates the deviation from the reference matrix, denoted as $\mathsf{IXX0}$ (Line~\ref{lin:alg-matinv-1}). Next, it checks the norm of $\mathsf{IXX0}$ to ensure that the Taylor expansion does not diverge as the order increases (Line~\ref{lin:alg-matinv-2}). Then, it computes the $\sf{k^{th}}$ order Taylor approximation $\sf{Xa}$, and sum them to obtain $\sf{k^{th}}$ order Taylor polynomial $\mathsf{Xp}$ (Line~\ref{lin:alg-matinv-3} - Line~\ref{lin:alg-matinv-4}), and computes an upper bound of the remainder norm $\mathsf{Eps}$ (Line~\ref{lin:alg-matinv-5}). Finally, the lower (upper) bound of matrix inverse $\mathsf{lXinv}$ ($\mathsf{uXinv}$) is achieved by subtracting (adding) $\sf{Eps}$ to $\sf{Xp}$ (Line~\ref{lin:alg-matinv-6} - Line~\ref{lin:alg-matinv-7}).

\begin{lemma}
\label{lem:mat-inv}
Given any non-singular matrix $\sf{X}$, the output of Algorithm~$\matinv$ $\langle\mathsf{lXinv},\mathsf{uXinv}\rangle$ satisfies that:
\[
    \mathsf{lXinv}\leq \inv(\sf{X})\leq \mathsf{uXinv}
\]
\end{lemma}


Lemma~\ref{lem:mat-inv} demonstrates that $\matinv$ produces valid bounds for the $\inv$ operation with fixed input. When applied to a linear set of input matrices, the linear bounds of $\sf{lXinv}$ and $\sf{uXinv}$ are derived through line-by-line propagation of linear relations. Together, these bounds define the overall linear bounds for the $\inv$ operation on the given input set. To mitigate computational overhead from excessive matrix multiplications, we make the following adjustments for tighter bounds while using a small $\sf{k}$: (1) select $\mathsf{X0}$ as the inverse of the center matrix within the input bounds; (2) iteratively divide the input range until the assertion at Line~\ref{lin:alg-matinv-2} is satisfied; and (3) increase $\mathsf{k}$ until $\mathsf{Eps}$ falls below the given tolerance. Example~\ref{exm:matinv-vs-direct} highlights the benefits of using $\matinv$ for matrix inversion. 


\begin{example}
\label{exm:matinv-vs-direct}
Given the constant lower and upper bounds of input matrices,
\begin{align*}
    \underline{X}=\begin{bmatrix}
    0.60 & -0.2 \\
    -0.02 & 0.90
    \end{bmatrix},\quad \overline{X}=\begin{bmatrix}
    0.90 & 0.02 \\ 
    0.02 & 1.30
\end{bmatrix},
\end{align*}
we combine lower bound of $\mathsf{lXinv}$ and upper bound of $\mathsf{uXinv}$ as overall bounds for $\inv$ operation, and take the norm difference $||\mathsf{uXinv} - \mathsf{lXinv}||$ 
as a measure of bound tightness. Using the adjugate method, we get $1.22$, while $\matinv$ improves this to $0.70$, which closely aligns with the empirical value of $0.66$.
\end{example}

\subsection{Abstracting Color Blending by Replacing Sorting Operator}
\label{sec:blending}

In $\pcsort$, the $\sort$ operation orders Gaussians by increasing depth so that those closer to the camera receive lower indices, enabling cumulative occlusion computation. However, we observe that instead of sorting, one can directly identify occluding Gaussians through pairwise depth comparisons.
Our algorithm, $\pcind$, iterates over Gaussian pairs and uses a boolean operator $\ind$ to select those with smaller depths relative to a given Gaussian. This approach avoids the complications associated with index-based accumulation (for example, the potential for counting the occlusion effect of the same Gaussian multiple times) and ultimately yields significantly tighter bounds on the pixel colors.

Algorithm \(\pcind\) takes as input the effective opacity \(\mathsf{a}\), Gaussian colors \(\mathsf{c}\), and Gaussian depths \(\mathsf{d}\) relative to the image plane, and outputs the pixel color \(\mathsf{pc}\). Its proceeds as follows: First, it computes the transparency between each pair of Gaussians, denoted as matrix \(\mathsf{V}\) (Line~\ref{lin:alg-pcind-1}). If the \(\sf{i^{th}}\) Gaussian is behind the \(\sf{j^{th}}\) ($\sf{d[i]-d[j]> 0}$), then $\ind(\sf{d[i]-d[j]})$ outputs $1$ and value of \(\mathsf{V[i,j] }\) is less than $1$, indicating that the contribution of the \(\sf{i^{th}}\) Gaussian's color is attenuated; otherwise, $\ind(\sf{d[i]-d[j]})$ outputs $0$ and value of \(\mathsf{V[i,j]}\) is set to be exact $1$, meaning the contribution of the \(\sf{i^{th}}\) Gaussian's color is unaffected by the \(\sf{j^{th}}\). Next, the combined transmittance for all Gaussians \(\mathsf{T}\) is computed (Line~\ref{lin:alg-pcind-2}). Finally, the pixel color \(\mathsf{pc}\) is obtained by aggregating the colors of all Gaussians, weighted accordingly (Line~\ref{lin:alg-pcind-3}). 


\begin{lemma}
\label{lem:sort-to-ind}
For any given inputs --- effective opacities $\sf{a}$, colors $\sf{c}$ and depth $\sf{d}$:
\[
\pcsort(\sf{a},\sf{c},\sf{d})=\pcind(\sf{a},\sf{c},\sf{d})
\]
\end{lemma}

Lemma~\ref{lem:sort-to-ind} demonstrates that $\pcsort$ and $\pcind$ are equivalent, i.e., replacing $\pcsort$ with $\pcind$ in $\gsplat$ yields same output values for same inputs. Furthermore, under input perturbation, $\pcind$ produces tighter linear bounds than $\pcsort$, as shown in Example~\ref{exm:3g-scene}.


\begin{example}
\label{exm:3g-scene}

Consider a scene with three Gaussians (red, green, blue) and a camera at the origin facing upward, with a perturbation of ±0.3 applied to each coordinate. At pixel (10,2) in a 20×20 image, the red channel’s upper bound is 0.6526 using $\pcsort$, 0.575 using $\pcind$, and 0.574 via empirical sampling. Similar results are observed for the green and blue channels. Figure~\ref{fig:sort-vs-ind} visualizes the upper bounds on the RGB channels for the entire image.

\end{example}

\begin{figure}[t]
    \centering
    \includegraphics[width=0.24\linewidth,height=2.5cm]{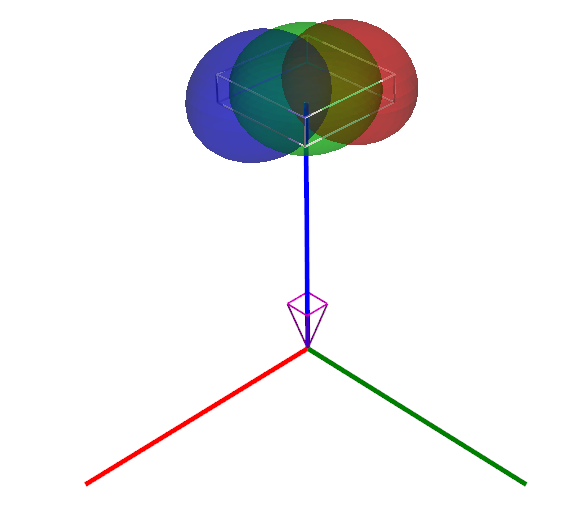}
    \includegraphics[width=0.24\linewidth,height=2.5cm]{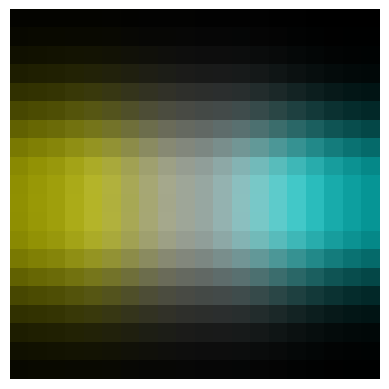}
    \includegraphics[width=0.24\linewidth,height=2.5cm]{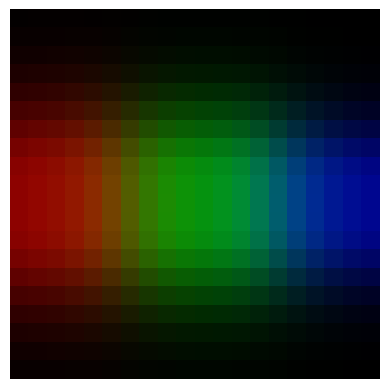}
    \includegraphics[width=0.24\linewidth,height=2.5cm]{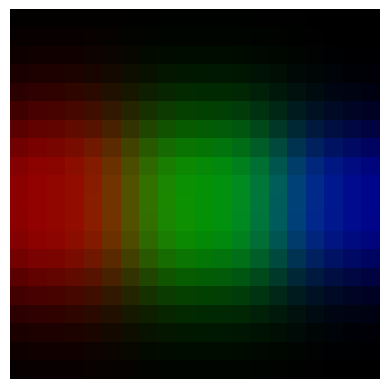}
    
    \caption{\small 
    Visualization of the scene and upper bounds from Example~\ref{exm:3g-scene}: left, the scene and camera; mid left, bounds computed with $\pcsort$;  mid right, bounds from $\pcind$; and right, empirical sampling bounds. Note that the bounds from \pcind\pcind are tighter (darker) and closely match the sampled bounds.
    }
    \label{fig:sort-vs-ind}
\end{figure}
\vspace{-0.2cm}

\subsection{Abstract Rendering Algorithm $\absrend$}
\label{sec:ab-rendering}

Our abstract rendering algorithm, $\absrend$, is derived from modifications to $\gsplat$, as follows:
\begin{description}
\item[(1)] All input arguments and variables in $\gsplat$ are replaced with their linear set-valued counterparts (with boldface font). 
\item[(2)] All the operations in Lines \ref{alg1:l2} - \ref{alg1:l7} and \ref{alg1:l9} - \ref{alg1:l10} are replaced by its corresponding linear relational operations. 
\item[(3)] Line~\ref{alg1:l8} is replaced to 
\[
\mathbf{Conic}\sf{[i]} \gets \matinv(\mat(\mathbf{Mp}\sf{[i]}, \mathbf{Mp}\sf{[i]}^\top); \mathsf{Conic0}, \mathsf{k}),
\] 
where both $\mathbf{Conic}\sf{[i]}$ and $\mathbf{Mp}\sf{[i]}$ are linear set-valued variables; and matrix $\mathsf{Conic0}$ and Taylor order $k$ are two parameters\footnote{matrix $\mathsf{Conic0}$ is selected as the inverse of center matrix from linear set of $\mat(\mathbf{Mp}\sf{[i]}, \mathbf{Mp}\sf{[i]}^\top)$, and $k$ defaulting to 8. A detailed explanation on selection for $\mathsf{Conic0}$ and $\mathsf{k}$ is provided in Section~\ref{sec:inv}.}.

\item[(4)] Line~\ref{alg1:l11} is modified to 
\[\mathbf{pc \gets \pcind(\mathbf{a}, \mathbf{c}, \mathbf{d})},
\]
where all involved variables are linear set-valued.

\item[(5)] The final output of $\absrend$ is a linear set of colors $\mathbf{pc}\subseteq [0,1]^3$ representing the set of possible colors for pixel $\sf{u}$. 
\end{description}
The complete abstract rendering algorithm computes the set of possible colors for every pixels in the image. The next theorem establishes the soundness of $\absrend$, which follows from the fact that all the operations involved in $\absrend$ are linear relational operations of linearly over-approximable functions.

\begin{theorem}
\label{thm:all}
$\absrend$ computes sound piecewise linear over-approximations of $\gsplat$ for any linear sets of inputs $\mathbf{Sc}$ and $\mathbf{C}$.
\end{theorem}

\begin{proof}
Let us fix the input linear sets of scenes $\mathbf{Sc}$, cameras $\mathbf{C}$, and a pixel $\mathsf{u}$. Consider any specific scene $\mathsf{Sc} \in \mathbf{Sc}$ and camera $\mathsf{C} \in \mathbf{C}$.
At Line~\ref{alg1:l2}, $\mathsf{uc}\mathsf{[i]} \in \mathbf{uc}\mathsf{[i]}$, because  Line~\ref{alg1:l2} in $\gsplat$ only contains linearly over-approximable operations (e.g. $\add$, $\mul$), and $\mathbf{uc[i]}$ are generated by their corresponding linearly relational operation. Similar containment relationship continues through to Line~\ref{alg1:l7}, where $\mathsf{d[i]} \in \mathbf{d}\mathsf{[i]}$.

At Line~\ref{alg1:l8}, by Lemma~\ref{lem:mat-inv}, we establish that $\mathsf{Conic[i]}\in[\mathsf{lXinv[i]}, \mathsf{uXinv[i]}]$. Since each operation involved in $\matinv$ is linearly over-approximable, it follows that $\mathsf{lXinv[i]}\in\mathbf{lXinv}\mathsf{[i]}$ and $\mathsf{uXinv[i]}\in\mathbf{uXinv}\mathsf{[i]}$. Thus, for the set $\mathbf{Xinv}\mathsf{[i]}$, defined as the union of $\mathbf{lXinv}\mathsf{[i]}$ and $\mathbf{uXinv}\mathsf{[i]}$, it implies  $\mathsf{Conic[i]}\in\mathbf{Xinv}\mathsf{[i]}$.

Similarly, at Line~\ref{alg1:l9} and Line~\ref{alg1:l10}, a corresponding containment relationship exists for the same reason as in Line~\ref{alg1:l2}. At Line~\ref{alg1:l11}, by Lemma~\ref{lem:sort-to-ind}, we have $\mathsf{pc}=\pcind(\mathsf{a,c,d})$, and since all operations involved in $\pcind$ are also linearly over-approximable, it follows that $\pcind(\mathsf{a,c,d})\in \mathbf{pc}$, and therefore, $\mathsf{pc}\in \mathbf{pc}$.

As the result of $\absrend$ over-approximates that of $\gsplat$ line by line, it follows that for any specific scene $\mathsf{Sc} \in \mathbf{Sc}$ and camera $\mathsf{C} \in \mathbf{C}$, we have $\mathsf{pc}\in\mathbf{pc}$, thus completing the proof.

\end{proof}

\section{Experiments with Abstract Rendering}
\label{sec:exp}
We implemented the abstract rendering algorithm $\absrend$ described in  Section~\ref{sec:abs-render}. 
As mentioned earlier, the linear approximation of the continuous operations in Table~\ref{tab:op} is implemented using $\crown$ \cite{10.5555/3495724.3495820}. 

\paragraph{$\absrend$ on GPUs.}
GPUs excel at parallelizing tasks that apply the same operation simultaneously across large datasets, such as numerical computations and rendering.
Consider a $\mathsf{W} \times \mathsf{H}$ pixel image of a scene with $\sf{N}$ Gaussians. 
A naive per-pixel implementation requires $\mathsf{WH}\times\mathsf{N}$ computations for effective opacities $\mathsf{a}$, while  $\pcind$ requires  $O(\mathsf{ WH}\times\mathsf{N}^2)$ operations. 
For realistic values ($\mathsf{N}>10k$), this can become prohibitive. 
Our implementation adopts a tile-based approach, similar to original Gaussian Splatting ~\cite{10.1145/3592433}, where the whole image is partitioned into small tiles
and the computation of effective opacity $\mathsf{a}$ is batched per tile, allowing independent parallel processing

Batching within each tile increases memory usage. Let $\mathsf{TS}$ denote the tile size and $\mathsf{BS}$ the batch size for processing Gaussians; then storing the linear bounds for $\mathsf{TS}^2$ pixels and $\mathsf{BS}$ Gaussians scales at $O(\mathsf{BS}\times\mathsf{TS}^2)$. Similarly, directly computing depth-based occlusion for all $\mathsf{N}^2$ Gaussian pairs in $\pcind$ are infeasible, so we batch the inner loop only, fixing a batch size for the outer loop to balance throughput and memory consumption.

Thus, the choices of tile size $\mathsf{TS}$ and batch size $\mathsf{BS}$ provide two tunable knobs for balancing performance and memory. Larger $\mathsf{BS}$ and $\mathsf{TS}$ yield faster computation but higher memory usage, while smaller values conserve memory at the cost of speed. By adjusting these parameters, users can optimize $\absrend$ for their available computing resources.


\subsection{Benchmarks and Evaluation Plan}
\label{sec:eval}
We evaluate $\absrend$ on six different Gaussian-splat scenarios chosen to cover a variety of scales and scene complexities.
These scenes are as follows: 
\texttt{BullDozer}~\cite{mildenhall2020nerf} is a scene with a LEGO bulldozer against an empty background. This scene makes it apparent how  camera pose changes influence the rendered image.
\texttt{PineTree} and \texttt{OSM} are scenes from~\cite{10745846}  originally described by triangular meshes. We converted them into Gaussian splats by training on images rendered from the mesh scenes. The resulting 3D Gaussian scenes are visually indistinguishable from the original meshes.
\texttt{Airport} is a large-scale scene  created from an airport environment in the Gazebo simulator.
\texttt{Garden} and \texttt{Stump}~\cite{10.1145/3592433} are real-world scenarios with large number of Gaussians.

All the scenes  were trained using the Splatfacto implementation of Gaussian splatting from Nerfstudio~\cite{nerfstudio}. For each scene, we generate multiple test cases \(\mathsf{Tc}\) by varying the nominal (unperturbed) camera pose. Table~\ref{tab:test_case} summarizes the details of these scenes.

\begin{table}[t]\centering

\caption{\small 
Gaussian splat scenes with their nominal (unperturbed) camera poses. For each scene, the camera pose is defined by the rotation $\mathsf{R}$ (given as XYZ Euler angles) and the translation $\mathsf{t}$. Multiple nominal camera poses are used for each scene.}
\label{tab:test_case}
\begin{tabular}{|l|l|l|l|l|l|}
\hline
$\mathsf{Sc}$ & $\mathsf{N}$ & $\mathsf{Tc}$ & $\mathsf{R}$ & $\mathsf{t}$ \\
\hline
\texttt{BullDozer} &56989 & \texttt{BD-1} &[-0.68,0.08,3.12] & [-0.3,5.2,2.7] \\
\hline
\multirow{3}{*}{\texttt{PineTree}} &\multirow{3}{*}{113368} & \texttt{PT-1} & [0.0,1.57,0.0] & [194.5, -0.1, 4.5]  \\
& & \texttt{PT-2} & [0.0,1.57,0.0] & [135.0, -0.1, 4.5] \\
& & \texttt{PT-3} & [0.0,0.0,0.0] & [80.0, 12.0, 40.0] \\
\hline
\multirow{3}{*}{\texttt{OSM}} &\multirow{2}{*}{144336} & \texttt{OSM-1} & [0.0,1.57,0.0] & [160.0,-0.1,20.0] \\
& &\texttt{OSM-2} & [0.0,1.57,0.0] & [160.0,-0.1,4.5]  \\
& &\texttt{OSM-3} & [0.0,-1.57,0.0] & [-160.0,0.0,10.0]  \\
\hline
\multirow{2}{*}{\texttt{Airport}} &\multirow{2}{*}{878217} & \texttt{AP-1} & [-1.71, -1.12, -1.73] & [-2914.7, 676.5, -97.5] \\
& & \texttt{AP-2} & [-1.36, -1.50, -1.36] & [-2654.9, 122.7, 26.7]  \\
\hline
\texttt{Garden} &524407 & \texttt{GD-1} & [-0.72, 0.31, 2.99] & [1.7, 3.7, 1.4]  \\
\hline
\texttt{Stump} &751333 & \texttt{ST-1} & [1.96, 0.10, 0.09] & [0.6, -1.8,-2.6] \\
\hline
\end{tabular}
\end{table}

We define two metrics to quantify the precision of abstract images: (1) Mean Pixel Gap ($\mpg$) 
\[
\mpg = \frac{1}{\mathsf{WH}}\sum_{j\in \mathsf{WH}}\Vert\mathsf{\overline{pc_j}}- \mathsf{\underline{pc_j}}\Vert_2
\]
computes the average Euclidean distance across the RGB channels between the upper and lower bound. This gives an overall measure of bound tightness across the entire image. (2) Max Pixel Gap ($\xpg$)
\[
\xpg = \max_{j\in\mathsf{WH}}\Vert\mathsf{\overline{pc_j}}- \mathsf{\underline{pc_j}}\Vert_2
\]
measures the largest pixel-wise $L_2$ norm indicating the worst-case losses in any single pixel's bound. 
A smaller $\mpg$ or $\xpg$ value corresponds to more precise abstract rendering.



\subsection{$\absrend$ is Sound}
\label{sec:exp:sound}

We applied $\absrend$ to the scenes in Table~\ref{tab:scenes} under various camera poses and perturbations. In these experiments, the input to $\absrend$ is a set of cameras ${\bf C}=B_{\mathsf{\epsilon_t}, \mathsf{\epsilon_R}}(\mathsf{C}_0)$, where $\mathsf{\epsilon_t}$ and $\mathsf{\epsilon_R}$ denote the perturbation radii for position and orientation, and $C_0$ is the nominal pose. We recorded two metrics, $\mpg$ and $\xpg$, and compared them against empirical estimates obtained by randomly sampling poses from ${\bf C}$. Different scenes were rendered at varying resolutions.
For some test cases, we further partitioned ${\bf C}$ into smaller subsets. We then computed bounds on the pixel colors for each subset and took their union. This practice improves the overall bound as the relaxation becomes tighter in each subset. 
Table~\ref{tab:scenes} summarizes the results from these experiments.


\begin{table}
\vspace{-0.7cm}
\caption{$\absrend$ scenes. Test case \(\mathsf{Tc}\), position perturbation \(\mathsf{\epsilon_t}\), orientation perturbation \(\mathsf{\epsilon_R}\), the number of partitions \(\mathsf{\#part}\), resolution \(\mathsf{res}\), runtime \(\mathsf{Rt}\), and metrics \(\mpg\) and \(\xpg\) under our method and via empirical sampling. Scalar \(\mathsf{\epsilon_t}\) and \(\mathsf{\epsilon_R}\) perturb all dimensions while vectors perturb only specific ones. }
\centering
\begin{tabular}{|l|l|l|l|l|l|l|l|l|l|}
\hline
\multicolumn{5}{|c|}{ } & \multicolumn{3}{|c|}{Ours}  & \multicolumn{2}{|c|}{Empirical}\\
\hline 
$\mathsf{Tc}$ & $\mathsf{\epsilon_t}$ &$\mathsf{\epsilon_R}$ &$\mathsf{\#part}$ & $\mathsf{res}$ & $\mathsf{Rt}$ & $\mpg$& $\xpg$ & $\mpg$& $\xpg$ \\
\hline
\texttt{PT-3} &  [2,0,0]& N/A&10&$72\times72$ & 246s & 0.27 & 1.73& 0.06 & 1.37\\ 
\texttt{PT-3} & [10,0,0]& N/A&100&$72\times 72$ & 48min & 0.47 & 1.73& 0.18 & 1.37\\ 
\texttt{PT-3} & [10,0,0]& N/A&200&$72\times72$ & 82min & 0.41 & 1.73& 0.18 & 1.37\\ 
\texttt{BD-1} & $0.1$ &N/A& 20 &$80\times 80$ &45min& 0.85 & 1.73 & 0.31 & 1.69 \\ 
\texttt{BD-1} & $0.2$ &N/A& 40 &$80\times 80$ & 74min & 1.01 & 1.73 & 0.67 & 1.67 \\ 
\texttt{BD-1} &N/A & [0,0,0.1] & 10 & $80\times 80$ & 22min & 0.51 &1.73& 0.22 & 1.57\\ 
\texttt{BD-1} &N/A & [0,0,0.3] & 60 & $80\times 80$ &  138min & 0.64 &1.73& 0.47 & 1.67\\ 
\texttt{AP-1} & $0.5$ &N/A& 1 & $160\times 160$ & 20min & 0.64&  1.72 & 0.07 &1.27\\ 
\texttt{AP-2} &N/A & $[0.001,0,0]$ & 1 & $160\times 160$ & 20min & 0.40 & 1.64 & 0.02 & 0.62\\ 
\texttt{GD-1} & $[0.05,0.0,0.0]$ &N/A& 1 & $100\times 100$ & 306s& 1.10 & 1.73 & 0.08 &0.83\\ 
\texttt{ST-1} & $0.002$ &N/A& 1 & $200\times 200$ & 10min & 0.12 & 1.53 & 0.01 &0.41\\ 
\texttt{OSM-3} & 0.01 &0.001& 1 & $48\times 48$ & 143s & 0.69 & 1.73 & 0.04 &1.01\\ 
\texttt{OSM-3} & 0.02 &0.002& 1 & $48\times48$ & 145s & 1.19 & 1.73 & 0.07 &1.20\\ 
\hline
\end{tabular}
\label{tab:scenes}
\vspace{-0.1cm}
\end{table}

First, pixel-color bounds produced by $\absrend$ are indeed sound. This is not surprising given Theorem~\ref{thm:all}, but given a good sanity check for our code. In all cases, the intervals computed strictly contain the empirically observed range of pixel colors from random sampling, confirming that our over-approximation is indeed conservative.


$\absrend$ handles a variety of 3D Gaussian scenes—from single-object (e.g., \texttt{BullDozer}) and synthetic (\texttt{PineTree}, \texttt{OSM}) to large-scale realistic scenes with hundreds of thousands of Gaussians (\texttt{Airport}, \texttt{Garden}, \texttt{Stump})—with reasonable runtime even at $200\times200$ resolution. Moreover, it accommodates both positional $(x,y,z)$ and angular (roll, pitch, yaw) camera perturbations.

Figure~\ref{fig:intro} visualizes the computed lower and upper pixel-color bounds for row 4 in Table~\ref{tab:scenes}, corresponding to a 10 cm horizontal shift in the camera position. Every pixel color from renderings within this range falls within our bounds. Our upper bound is brighter and the lower bound darker compared to the empirical bounds. The over-approximation is largely due to repeated $\mul$ operations in $\pcind$.


As the size of ${\bf C}$ increases, the pixel-color bounds naturally slacken. For example, in rows 12–13 of Table~\ref{tab:scenes}, a larger perturbation range yields higher $\mpg$ and $\xpg$ values and a wider gap from empirical estimates. Conversely, partitioning the input into finer subsets tightens the bounds: in rows 2–3 (\texttt{PT-3}), increasing partitions from 100 to 200 noticeably reduces $\mpg$, and a similar trend is observed in rows 6–7 for yaw perturbations.


Overall, these experiments confirm that our method yields sound over approximations for $\gsplat$ and can flexibly trade off runtime for tighter bounds by adjusting input partition granularity.

\subsection{Comparing with Image Abstraction Method of~\cite{10745846}}
\label{sec:exp:comp}
We compare $\absrend$ with the baseline approach from \cite{10745846}, which computes interval images from triangular meshes under a set of camera positions. 
We selected four test cases from the scenes in \cite{10745846} and introduced various sets of cameras ${\bf C}$. We recorded two metrics, $\mpg$ and $\xpg$, under both our approach and the baseline. To match with ~\cite{10745846}, all images are rendered at $49\times 49$ pixels. Table~\ref{tab:comparison} summarizes our results.

\begin{table}[t]
\caption{\small Comparison of~\cite{10745846} and our method. Test case~$\mathsf{Tc}$, runtime $\mathsf{Rt}$ for our method and baseline and the two metrics measured for our method and baseline.} 
\centering
\begin{tabular}{|l|l|l|l|l|l|l|l|l|}
\hline
\multicolumn{3}{|c|}{ } & \multicolumn{3}{|c|}{Ours} & \multicolumn{3}{|c|}{Baseline~\cite{10745846}}\\ 
\hline 
$\mathsf{Tc}$ & $\mathsf{\epsilon_t}$  & $\mathsf{\epsilon_R}$ &$\mathsf{Rt}$ & $\mpg$ & $\xpg$ & $\mathsf{Rt}$ & $\mpg$ & $\xpg$\\ 
\hline
\texttt{PT-1} &  0.01  & N/A& 125s & 0.07 & 0.74 & 222s & 0.05 & 1.25\\ 
\texttt{PT-1} & 0.02  & N/A& 124s & 0.24 & 1.54 & 251s& 0.05 & 1.25\\ 
\texttt{PT-1} & N/A  & $[0,0,0.01]$& 67s & 0.33 & 1.72 & N/A & N/A & N/A\\
\texttt{PT-2} &  0.01& N/A& 84s & 0.07 & 1.36  & 110s & 0.03 & 1.62\\
\texttt{PT-2} & 0.02 & N/A& 88s & 0.14 & 1.60 & 120s & 0.03 & 1.62\\
\texttt{PT-2} & N/A &$[0,0,0.01]$ & 35s & 0.08 &  1.52 & N/A & N/A & N/A\\
\hline
\texttt{OSM-1} & 0.01 & N/A& 105s & 0.11 & 0.66 & 26min & 0.11 & 1.59\\
\texttt{OSM-1} & 0.03 & N/A& 111s & 0.32 & 1.54 & 29min & 0.14 & 1.65 \\
\texttt{OSM-1} & N/A & $[0,0.001,0]$& 53s & 0.51 & 1.70 & N/A & N/A & N/A\\
\texttt{OSM-2} & 0.01 & N/A& 108s & 0.11 & 0.84 & 26min & 0.10 & 1.60\\
\texttt{OSM-2} & 0.03 & N/A& 104s & 0.30 & 1.63 & 29min & 0.12 & 1.65\\
\texttt{OSM-2} & N/A& $[0,0.001,0]$ & 59s & 0.51 & 1.73 & N/A & N/A & N/A\\
\hline
\end{tabular}
\label{tab:comparison}
\end{table}

Overall, our algorithm achieves faster runtime than the baseline in most tests. For \texttt{PT-1}, our method is about 2$\times$ faster, while in the \texttt{OSM} cases, it is roughly 14$\times$ faster. In terms of $\xpg$, our approach often yields tighter worst-case bounds, especially with small input perturbations, indicating a more precise upper bound on pixel colors. However, as the perturbation grows, our $\mpg$ can become larger (i.e., less tight) compared to the baseline. However, because our method runs faster, we can split ${\bf C}$ into smaller subsets to refine these bounds when needed, as discussed in Section~\ref{sec:exp:sound}.
Lastly, our method also supports perturbations in camera orientation, whereas the baseline does not handle those. This flexibility allows us to cover a broader range of real-world camera variations while retaining reasonable performance. Note that the rendering algorithm discussed in~\cite{10745846} could potentially be abstracted by our method since it can be written using operations in Table~\ref{tab:op}.

\subsection{Scalability of $\absrend$}

We evaluate the scalability of $\absrend$ by varying the tile size ($\mathsf{TS}$), batch size ($\mathsf{BS}$), and image resolution ($\mathsf{res}$) while keeping the same scene and camera perturbations (see Table~\ref{tab1}). The table reports the number of Gaussians processed and peak GPU memory usage.

Increasing the tile size from $4\times4$ to $8\times8$ reduces computation time but increases memory usage roughly quadratically with the tile dimension. Similarly, as $\mathsf{BS}$ increases, runtime drops while memory consumption grows proportionally; however, beyond a certain batch size—roughly the number of Gaussians per tile—the performance gains diminish, and too small a batch size negates memory savings due to overhead in $\pcind$. Additionally, processing more Gaussians and higher resolutions linearly increase both runtime and memory demands.

Overall, these results highlight a clear trade-off between speed and memory usage. By adjusting $\mathsf{TS}$, $\mathsf{BS}$, and $\mathsf{res}$, users can balance performance and resource requirements to best suit their hardware.

\begin{table}[t]
\caption{\small Scalability result for \texttt{PT-1} Under Varying Tile and Batch Sizes.}\label{tab1}
\centering
\begin{tabular}{|l|l|l|r|r|r|r|r|}
\hline 
$\mathsf{Tc}$ & $\mathsf{res}$ & $\mathsf{TS}$ & $\mathsf{BS}$ & $\mathsf{\# Gauss}$ & $\mathsf{Rt}$ & GPU Mem (MB) \\ 
\hline
\texttt{PT-1} & $48\times 48$ & $4\times 4$ & 10000 &317243& 283s & 2166 \\ 
\texttt{PT-1} & $48\times 48$ & $8\times 8$ & 2000  &189054& 256s & 9968 \\ 
\texttt{PT-1} & $48\times 48$ & $8\times 8$ & 4000  &189054& 162s& 10838\\ 
\texttt{PT-1} & $48\times 48$ & $8\times 8$ & 10000 &189054& 110s & 14856 \\ 
\texttt{PT-1} & $48\times 48$ & $8\times 8$ & 30000 &189054& 91s & 21452 \\ 
\texttt{PT-1} & $48\times 48$ & $8\times 8$ & 30000 &216671& 107s & 19384 \\ 
\texttt{PT-1} & $96\times 96$ & $8\times 8$ & 30000 &232509& 256s & 12694 \\ 
\hline
\end{tabular}
\vspace{-0.1cm}
\end{table}

\subsection{Long Spikey Guassians Cause Corase Analysis}
An interesting observation from our experiments is that long, narrow Gaussians in a scene can lead to significantly more conservative bounds in our method. Figure~\ref{fig:thin_gauss} illustrates this phenomenon using the same test case (\texttt{BD-1}) under a small yaw perturbation of -0.005-0 rad.

In the left images, the bounds are visibly too loose because the scene contains many “spiky” Gaussians with near-singular covariance matrices. Even small perturbations in such covariances—when inverting them—can cause very large ranges in the conic form as discussed in Section~\ref{sec:inv}. For example, consider one Gaussian in the image coordinate frame with tight covariance bounds:
\[
\small
\underline{\mathsf{Cov}}=\begin{bmatrix}
    0.3158& 0.3203\\
    0.3203 & 0.3500
\end{bmatrix},
\overline{\mathsf{Cov}}=\begin{bmatrix}
    0.3196 & 0.3226 \\
    0.3226 & 0.3508
\end{bmatrix}
\]
After inversion, this becomes:
\[
\small
\underline{\mathsf{Conic}}=\begin{bmatrix}
    37.9340 & -47.8139\\
    -47.8381 &  34.5632
\end{bmatrix},
\overline{\mathsf{Conic}}=\begin{bmatrix}
    51.9298 & -34.7045\\
    -34.6626 &  46.9024
\end{bmatrix}
\]
Such drastic changes in the inverse bleeds into large values from the Gaussians, overestimated effective opacities $\mathsf{a}$, and ultimately, loose pixel-color bounds.

To mitigate this issue, we retrained the scene while penalizing thin Gaussians to avoid near-singular covariances. The resulting bounds (shown in the right images of Fig.~\ref{fig:thin_gauss}) are significantly tighter. Thus, the structure of the 3D Gaussian scene itself can substantially impact the precision of our analysis, and regularizing or penalizing long, thin Gaussians provides a practical way to reduce over-approximation.

\begin{figure}[t]
    \centering
    \includegraphics[width=0.24\linewidth,height=2.1cm]{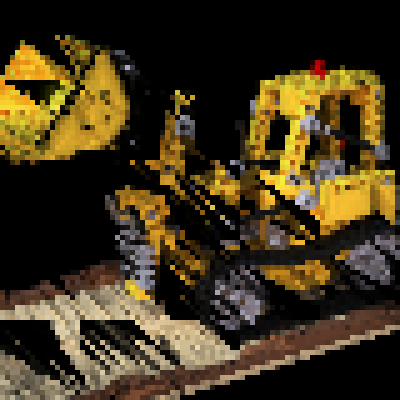}
    \includegraphics[width=0.24\linewidth,height=2.1cm]{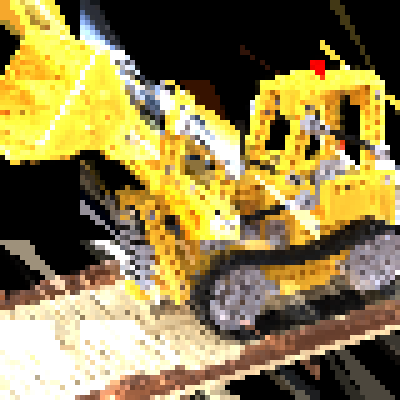}
    \includegraphics[width=0.24\linewidth,height=2.1cm]{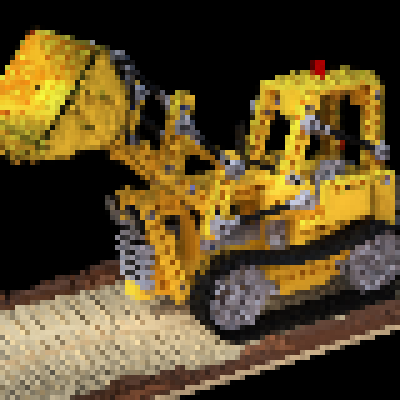}
    \includegraphics[width=0.24\linewidth,height=2.1cm]{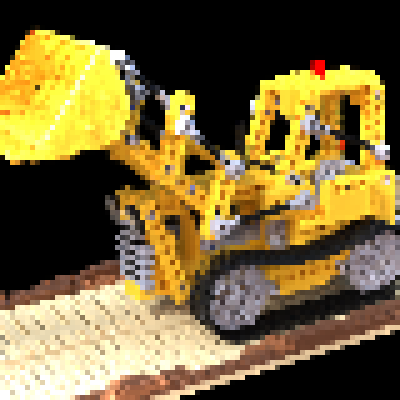}
    \caption{\small {\em Left $\&$ center left:} Lower and upper pixel-color bounds before penalizing thin Gaussians. {\em Center right  $\&$ right:} Bounds after penalizing thin Gaussians.}
    \label{fig:thin_gauss}
\end{figure}

\subsection{Abstract Rendering Sets of Scene}
While prior experiments focused on camera pose uncertainty, $\absrend$ also supports propagating sets of 3D Gaussian scene parameters (e.g., color, mean position, opacity). This capability enables analyzing environmental variations or dynamic scenes represented as evolving Gaussians. We validate this with three experiments on the \texttt{PT-1}, fixing the camera pose and perturbing specific Gaussian attributes.

\begin{figure}[htbp]
    \centering
    \begin{minipage}[t]{0.3\textwidth}
        \centering
        \begin{subfigure}[b]{\textwidth}
            \centering
            \includegraphics[width=0.80\linewidth,height=1.7cm]{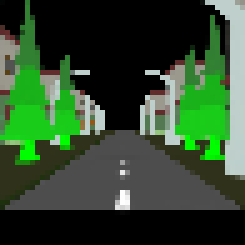}
        \end{subfigure}
        

        \begin{subfigure}[b]{\textwidth}
            \centering
            \includegraphics[width=0.80\linewidth,height=1.7cm]{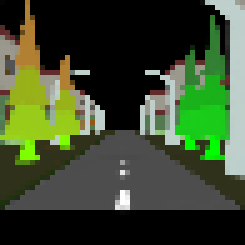}
        \end{subfigure}
    \end{minipage}
    \hfill
    \begin{minipage}[t]{0.3\textwidth}
        \centering
        \begin{subfigure}[b]{\textwidth}
            \centering
            \includegraphics[width=0.80\linewidth,height=1.7cm]{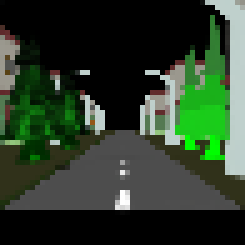}
        \end{subfigure}
        

        \begin{subfigure}[b]{\textwidth}
            \centering
            \includegraphics[width=0.80\linewidth,height=1.7cm]{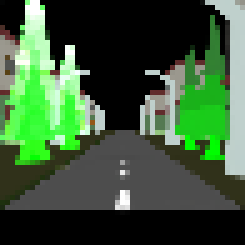}
        \end{subfigure}
    \end{minipage}
    \hfill
    \begin{minipage}[t]{0.3\textwidth}
        \centering
        \begin{subfigure}[b]{\textwidth}
            \centering
            \includegraphics[width=0.80\linewidth,height=1.7cm]{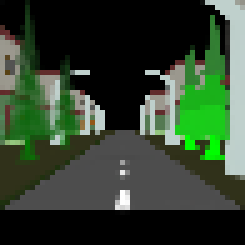}
        \end{subfigure}
        \begin{subfigure}[b]{\textwidth}
            \centering
            \includegraphics[width=0.80\linewidth,height=1.7cm]{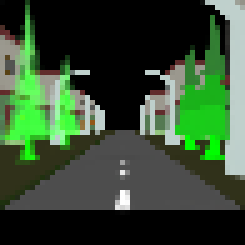}
        \end{subfigure}
    \end{minipage}
    \caption{\small Lower (Top) and upper (Bottom) bound under Gaussian color (Left), mean (Mid) and opacity (Right) perturbation.}
    \label{fig:ptb_scene}
\end{figure}

In the first experiment, we perturb the red channel of the Gaussians representing the two trees by 0–0.5. The resulting pixel-color bounds (Fig.\ref{fig:ptb_scene} Left) fully encompass the empirical range, confirming that our method captures chromatic uncertainty. In the second experiment, a 1 m perturbation along the +y-axis is applied to the tree Gaussians' mean positions; the bounds (Fig.\ref{fig:ptb_scene} Mid) capture both spatial displacement and occlusion effects, demonstrating that $\absrend$ effectively propagates geometric uncertainty. Finally, by scaling the original opacities by 0.1 (capping them) and adding a +0.1 perturbation, we test transparency variations. The resulting bounds (Fig.~\ref{fig:ptb_scene} Right) remain sound, though somewhat looser due to multiplicative interactions in $\pcind$.




These experiments highlight the flexibility of $\absrend$: by abstracting Gaussian parameters, we can model environmental dynamics such as object motion, material fading, or scene evolution. This capability aligns with recent work on dynamic 3D Gaussians~\cite{Wu_2024_CVPR,yang2023gs4d,xu2024splatfactownerfstudioimplementationgaussian} and underscores the potential for robustness analysis in vision systems operating in non-static environments. 





\section{Discussion and Future Work}
We presented the first method for rigorously bounding the output of 3D Gaussian splatting under camera pose uncertainty, enabling formal analysis of vision-based systems operating in dynamic environments. By reformulating depth-sorting as an index-free blending process and bounding matrix inverses via Taylor expansions, our approach avoids combinatorial over-approximation while maintaining computational tractability. Experiments demonstrate scalability to scenes with over 750k Gaussians. 


Potential directions include refining the blending process to reduce error accumulation, exploring higher-order bounds or symbolic representations for pixel correlations, and extending the framework to closed-loop system verification. Adapting our approach to other rendering paradigms (e.g., NeRF, mesh-based rendering) could further unify uncertainty-aware analysis across computer vision pipelines. It is worth further investigating abstract rendering for scene variations.






%
%
%
%
\bibliographystyle{splncs04}
\bibliography{egbib}

\appendix

\section{Appendix}

\subsection{Proofs of Lemmas used for Abstract Rendering}
\label{app:proof}

\begin{proposition}
Any continuous function is linearly over-approximable.
\end{proposition}


\begin{proof}
We present a simple proof for scalar $f$ and this can be extended to higher dimensions. 
We fix $x_0 \in \reals$ and  a finite slope $k\in\reals$. Since $f(x)-k(x-x_0)$ is continuous, for any $\varepsilon>0$, there must exist an open neighborhood $U(x_0)$ around $x_0$, such that $\forall x\in R$, \begin{align}
   |f(x)-f(x_0)-k(x-x_0)|<\frac{\varepsilon}{3}.
\end{align} 
Equivalently,
\begin{align}
   f(x_0)+k(x-x_0)-\frac{\varepsilon}{3}<f(x)<f(x_0)+k(x-x_0)+\frac{\varepsilon}{3}.
\end{align}
We construct candidate  linear upper and lower bounds over $U$ as: 
\begin{align}
    & l_U(x)=k(x-x_0)+f(x_0)-\frac{\varepsilon}{3},\quad
    u_U(x)=k(x-x_0)+f(x_0)+\frac{\varepsilon}{3}.
\end{align} 
We can check that these linear functions $l_U(x)$ and $u_U(x)$ are valid lower and upper bounds for $f(x)$ over $U$, because $\forall x\in U(x_0)$:
\begin{align}
    & f(x)-l_U(x)=f(x)-
    (k(x-x_0)+f(x_0)-\frac{\varepsilon}{3})
    >
    -\frac{\varepsilon}{3}+\frac{\varepsilon}{3}= 0,\\
    & f(x)-u_U(x)=f(x)-(k(x-x_0)+f(x_0)+\frac{\varepsilon}{3})<\frac{\varepsilon}{3}-\frac{\varepsilon}{3}= 0.
\end{align}
Additionally, difference between $l_U(x)$ and $u_U(x)$ is tightly bounded, as $\forall x\in U(x_0)$:
\begin{align*}
    |u_U(x)-l_U(x)|=(k(x-x_0)+f(x_0)+\frac{\varepsilon}{3})-(k(x-x_0)+f(x_0)-\frac{\varepsilon}{3})=\frac{2\varepsilon}{3}<\varepsilon.
\end{align*}
Now, for any compact set $B$, the collection $\mathcal{U} = \{U(x) \mid x\in B\}$ forms an open cover of $B$. By the definition of compactness, there exists a finite subcover of $B$, denoted by $\{U_i = U(x_i)\mid i = 1,\dots, s\}$. We can now define piecewise linear functions on $B$ as follows:
\begin{equation}
\begin{aligned}    
    u_B(x) &= \min_{j \in I}u_{U_j}(x),\\
    l_B(x) &= \max_{j \in I}l_{U_j}(x),\quad if\ x\in \{U_j\mid j \in I\}
\end{aligned}
\end{equation}
where $I$ is the index set of $U_i$ that cover $x$. Consequently, we have
\begin{equation}  
    u_B(x) =\min_{j \in I}u_{U_j}(x) \geq f(x) \geq \max_{j \in I}l_{U_j}(x) = l_B(x),
\end{equation}
and
\begin{equation}
    |u_B(x) - l_B(x)| \leq u_{U_j}(x) - l_{U_j}(x)\leq \varepsilon,\quad \forall j \in I.
\end{equation}
Thus, $f$ is linearly over-approximable.

\qed
\end{proof}

\begin{lemma}
Given any non-singular matrix $\sf{X}$, the output of Algorithm~$\matinv$ $\langle\mathsf{lXinv},\mathsf{uXinv}\rangle$ satisfies that:
\[
    \mathsf{lXinv}\leq \inv(\sf{X})\leq \mathsf{uXinv}
\]
\end{lemma}


\begin{proof}
For any non-singular $n\times n$ matrices $\sf{X}$ and $\sf{X_{ref}}$, denote their difference as $\sf{\Delta X}=\sf{X}-\sf{X_{ref}}$. The $\sf{k^{th}}$ order Taylor Polynomial $\sf{Xp}$ and remainder $\sf{R}$ of matrix inverse of $\sf{X}$, estimated at $\sf{X_{ref}}$, can be written as follows:
\begin{align}
    & \label{eq:k-TM} \mathsf{Xp}=\mathsf{X_{ref}^{-1}}\sum_{\mathsf{i}=0}^{\mathsf{k}}(\mathsf{\Delta X}\cdot \mathsf{X_{ref}^{-1}})^\mathsf{i}\\
    & \label{eq:R-TM} \mathsf{R}=\mathsf{X_{ref}^{-1}}\sum_{\mathsf{i}=\mathsf{k}+1}^{\infty}(\mathsf{\Delta X}\cdot \mathsf{X_{ref}^{-1}})^\mathsf{i}
\end{align}
Assuming that $\Vert\mathsf{\Delta X}\cdot \mathsf{X_{ref}^{-1}}\Vert<1$, the remainder $\sf{R}$ can be bounded by:
\begin{align}
\label{eq:bound-R}
\begin{split}
    \mathsf{Eps}=\Vert\sf{R}\Vert
    & \leq \Vert\mathsf{X_{ref}^{-1}}\Vert\cdot \sum_{\mathsf{i}=\mathsf{k}+1}^{\infty}\Vert\mathsf{\Delta X}\cdot \mathsf{X_{ref}^{-1}}\Vert^\mathsf{i}
    =\Vert\mathsf{X_{ref}^{-1}}\Vert\cdot \frac{\Vert\mathsf{\Delta X}\cdot \mathsf{X_{ref}^{-1}}\Vert^{\mathsf{k}+1}}{1-\Vert\mathsf{\Delta X}\cdot \mathsf{X_{ref}^{-1}}\Vert}
\end{split}
\end{align}
Using the bound in inequality~\ref{eq:bound-R}, we obtain bounds on $\sf{X^{-1}}$: 
\begin{align}
    \mathsf{Xp}-\mathsf{Eps}\leq \mathsf{X^{-1}} \leq \mathsf{Xp}+\mathsf{Eps}
\end{align}
Thus, we establish lower and upper bound for $\sf{X^{-1}}$ as:
\begin{align}
\label{eq:lX-uX}
    \sf{lXinv}=\sf{Xp}-\sf{Eps},\quad \sf{uXinv}=\sf{Xp}+\sf{Eps}.
\end{align}

Finally, by replacing $\sf{X_{ref}^{-1}}$ with $\sf{X0}$ in the expression for $\sf{Xp}$ (in Equation~\ref{eq:k-TM}), $\sf{Eps}$ (in Equation~\ref{eq:bound-R}), $\sf{lXinv}$ and $\sf{uXinv}$ (in Equation~\ref{eq:lX-uX}), we obtain the same expressions as those given in Algorithm $\matinv$.
\qed
\end{proof}

\begin{lemma}
For any given inputs --- effective opacities $\sf{a}$, colors $\sf{c}$ and depth $\sf{d}$, it follows:
\[
\pcsort(\sf{a},\sf{c},\sf{d})=\pcind(\sf{a},\sf{c},\sf{d})
\]
\end{lemma}


\begin{proof}
We first derive the math expression of pixel color $\sf{pc}$ from $\pcsort$.
By denoting $\argsort(\sf{d})$ as a mapping $\sigma:\sf{I}\to\sf{I}$, where $\sf{I}$ refers to the index set of Gaussians, $\sf{as}$ at Line~\ref{lin:alg-pcsort-1}, $\sf{cs}$ at Line~\ref{lin:alg-pcsort-2} and $\sf{vs}$ at Line~\ref{lin:alg-pcsort-3} can be written as follows.
\begin{align}
    & \label{algn:as} \mathsf{as}_\mathsf{i}=\mathsf{a}_\mathsf{\sigma(i)}\\
    & \label{algn:cs} \mathsf{cs}_\mathsf{i}=\mathsf{c}_\mathsf{\sigma(i)}\\
    & \label{algn:bs} \mathsf{vs}_\mathsf{i}=1-\mathsf{a}_\mathsf{\sigma(i)}
\end{align}
According to the definition of $\ind$ operation, we have:
\begin{align}
    \label{eq:ind-def}
    \ind(\mathsf{d}_\mathsf{i}-\mathsf{d}_\mathsf{j})=\left\{
    \begin{matrix}
        1 & \mbox{ if } \mathsf{d}_\mathsf{i}>\mathsf{d}_\mathsf{j}\\
        0 & \mbox{ if } \mathsf{d}_\mathsf{i}\leq \mathsf{d}_\mathsf{j}.
    \end{matrix}\right.
\end{align}
Then by replacing $\sf{i,j}$ with $\sf{\sigma(i),\sigma(j)}$, and multiplying $\sf{a}_{\sigma(\sf{i})}$, (\ref{eq:ind-def}) becomes:
\begin{align}
    \label{eq:a-ind}
    \mathsf{a}_{\sigma(\mathsf{i})}\cdot \ind(\mathsf{d}_{\sigma(\mathsf{i})}-\mathsf{d}_{\sigma(\mathsf{j})})=\left\{
    \begin{matrix}
        \mathsf{a}_{\sigma(\mathsf{i})} & \mbox{ if } \mathsf{d}_{\sigma(\mathsf{i})}>\mathsf{d}_{\sigma(\mathsf{j})}\\
        0 & \mbox{ if } \mathsf{d}_{\sigma(\mathsf{i})}\leq \mathsf{d}_{\sigma(\mathsf{j})}
    \end{matrix}\right.
\end{align}
For any fixed index  $\sf{i}$, by multiplying all the case in the second branch of ~(\ref{eq:a-ind}), it follows:
\begin{align}
\label{eq:prod-a}
    \prod_{\mathsf{j}=\mathsf{i}}^{N}(1-\mathsf{a}_\mathsf{\sigma(i)}\cdot\ind(\mathsf{d}_{\sigma(\mathsf{i})}-\mathsf{d}_{\sigma(\mathsf{j})}))=1.
\end{align}
By applying Equation~\ref{algn:bs}, Equation~\ref{eq:a-ind} and Equation~\ref{eq:prod-a} into Line~\ref{lin:alg-pcsort-4}, $\sf{Ts}_{\sf{i}}$ can be written as follows:
\begin{align}
\label{eq:Ts}
\begin{split}
    \mathsf{Ts}_\mathsf{i}
    & = \prod_{\mathsf{j}=1}^\mathsf{i-1} (1-\mathsf{vs}_\mathsf{i}) = \prod_{\mathsf{j}=1}^\mathsf{i-1} (1-\mathsf{a}_\mathsf{\sigma(i)})\\
    & = \prod_{\mathsf{j}=1}^{i-1}(1-\mathsf{a}_\mathsf{\sigma(i)}\cdot\ind(\mathsf{d}_{\sigma(\mathsf{i})}-\mathsf{d}_{\sigma(\mathsf{j})}))\\
    & = \prod_{\mathsf{j}=1}^{N}(1-\mathsf{a}_\mathsf{\sigma(i)}\cdot\ind(\mathsf{d}_{\sigma(\mathsf{i})}-\mathsf{d}_{\sigma(\mathsf{j})}))\\
\end{split} 
\end{align}
Then by applying Equation~\ref{eq:Ts} into Line~\ref{lin:alg-pcsort-5}, the pixel color $\sf{pc}$ computed via Algorithm~$\pcsort$ can be expressed as:
\begin{align}
\label{eq:pc-sort-math}
\begin{split}
    \mathsf{pc}=
    & \sum_{\mathsf{i}=1}^\mathsf{N} \left(\prod_{\mathsf{j}=1}^{N}(1-\mathsf{a}_\mathsf{\sigma(i)}\cdot\ind(\mathsf{d}_{\sigma(\mathsf{i})}-\mathsf{d}_{\sigma(\mathsf{j})}))\mathsf{a}_\mathsf{\sigma(\mathsf{i})}\mathsf{c}_\mathsf{\sigma(\mathsf{i})}\right)
\end{split}
\end{align}

Since the summation and product operations are invariant to the reordering of input elements, and both indices $\sf{i}$ and $\sf{j}$ iterate over the entire index set $\mathsf{I}$, the reordering mapping $\sigma$ in Equation~\ref{eq:pc-sort-math} can be eliminated, resulting in:
\begin{align}
\label{eq:pc-left}
\begin{split}
    \mathsf{pc}=
    & \sum_{\mathsf{i}=1}^\mathsf{N} \left(\prod_{\mathsf{j}=1}^{N}(1-\mathsf{a}_\mathsf{i}\cdot\ind(\mathsf{d}_{\mathsf{i}}-\mathsf{d}_{\mathsf{j}}))\mathsf{a}_\mathsf{\mathsf{i}}\mathsf{c}_\mathsf{\mathsf{i}}\right)
\end{split}
\end{align}

Next, we derive the math expression for $\sf{pc}$ in Algorithm~$\pcind$. $\sf{V_{ij}}$ at Line~\ref{lin:alg-pcind-1} and $\sf{T_i}$ at Line~\ref{lin:alg-pcind-2} can be written as:
\begin{align}
    & \label{algn:V} \mathsf{V_{ij}}=1-\mathsf{a_j}\cdot \ind(\mathsf{d_i}-\mathsf{d_j})\\
    & \label{algn:T} \mathsf{T}_\mathsf{i}
    = \prod_{\mathsf{j}=1}^{N}(1-\mathsf{a}_\mathsf{i}\cdot\ind(\mathsf{d}_{\mathsf{i}}-\mathsf{d}_{\mathsf{j}}))
\end{align}

Then, by applying Equation~\ref{algn:T} into Line~\ref{lin:alg-pcind-3} of Algorithm~$\pcind$, the expression for $\mathsf{pc}$ becomes:
\begin{align}
\label{eq:pc-right}
    \mathsf{pc}=
    & \sum_{\mathsf{i}=1}^\mathsf{N} \left(\prod_{\mathsf{j}=1}^{N}(1-\mathsf{a}_\mathsf{i}\cdot\ind(\mathsf{d}_{\mathsf{i}}-\mathsf{d}_{\mathsf{j}}))\mathsf{a}_\mathsf{\mathsf{i}}\mathsf{c}_\mathsf{\mathsf{i}}\right)
\end{align}
Note that the expressions for $\mathsf{pc}$ are identical in both Equation~\ref{eq:pc-left} and Equation~\ref{eq:pc-right}. This demonstrates that Algorithms $\pcsort$ and $\pcind$ yield the same input-output relationship while applying different operations, thus completing this proof.
\qed
\end{proof}

\end{document}